\pdfoutput=1

\documentclass[review,11pt]{elsarticle}

\usepackage[margin=2.5cm]{geometry}

\usepackage[utf8]{inputenc}

\usepackage{graphicx}
\usepackage{caption}
\usepackage{subfig}
\usepackage{multirow}
\usepackage{url}
\usepackage{xcolor}

\usepackage{amssymb}
\usepackage{array}

\usepackage{amsthm}
\usepackage{amsmath}
\usepackage{algorithm}
\usepackage{algorithmic}
\usepackage{bbm}
\usepackage{makecell}

\newtheorem{theor}{Theorem}

\usepackage{enumitem}
\newlist{Properties}{enumerate}{2}
\setlist[Properties]{label=Property \arabic*., font=\textbf, itemindent=*}

\begin{document}

\begin{frontmatter}{}

\title{A $k$-additive Choquet integral-based approach to approximate the SHAP values for local interpretability in machine learning}

\author[aff1,aff2]{Guilherme Dean Pelegrina\corref{fn1}}
\ead{guidean@unicamp.br}
\author[aff1]{Leonardo Tomazeli Duarte}
\ead{leonardo.duarte@fca.unicamp.br}
\author[aff2,aff3]{Michel Grabisch}
\ead{michel.grabisch@univ-paris1.fr}

\cortext[fn1]{Corresponding author}

\address[aff1]{School of Applied Sciences - University of Campinas, Limeira, Brazil}
\address[aff2]{Centre d’Économie de la Sorbonne - Université Paris I Panthéon-Sorbonne, Paris, France}
\address[aff3]{Paris School of Economics - Université Paris I Panthéon-Sorbonne, Paris, France}

\begin{abstract}
 Besides accuracy, recent studies on machine learning models have been addressing the question on how the obtained results can be interpreted. Indeed, while complex machine learning models are able to provide very good results in terms of accuracy even in challenging applications, it is difficult to interpret them. Aiming at providing some interpretability for such models, one of the most famous methods, called SHAP, borrows the Shapley value concept from game theory in order to locally explain the predicted outcome of an instance of interest. As the SHAP values calculation needs previous computations on all possible coalitions of attributes, its computational cost can be very high. Therefore, a SHAP-based method called Kernel SHAP adopts an efficient strategy that approximate such values with less computational effort. In this paper, we also address local interpretability in machine learning based on Shapley values. Firstly, we provide a straightforward formulation of a SHAP-based method for local interpretability by using the Choquet integral, which leads to both Shapley values and Shapley interaction indices. Moreover, we also adopt the concept of $k$-additive games from game theory, which contributes to reduce the computational effort when estimating the SHAP values. The obtained results attest that our proposal needs less computations on coalitions of attributes to approximate the SHAP values.

\end{abstract}


\end{frontmatter}{}

\section{Introduction}
\label{sec:intro}

In the last decade, Machine Learning (ML) models have been used to deal with problems that directly affect people's life, such as consumer credit scoring~\citep{Kruppa2013}, cybersecurity~\citep{Xin2018}, disease detection~\citep{Ahsan2022} and patient care evaluation~\citep{BenIsrael2020}. Aiming at dealing with such problems, complex ML models have been proposed to achieve good solutions in terms of accuracy. Examples include random forests~\citep{Fawagreh2014,Biau2016}, deep neural networks~\citep{LeCun2015,Goodfellow2016} and gradient boosting algorithms~\citep{Bentejac2021}. Despite the good performance in terms of accuracy, these models act as black box models, since the obtained results (predictions and/or classifications) are difficult to be interpreted. Therefore, there is an inherent trade-off between adopting an accurate model, whose structure is frequently complex, or an interpretable model, such as linear/logistic regression~\citep{Molnar2021}.

Interpretability plays an important role in machine learning-based automatic decisions and has been discussed in several recent works in the ML community~\citep{Lipton2018,Gilpin2018,Carvalho2019,Molnar2021,Setzu2021}. As stated by~\citet{Miller2019}, interpretability can be defined as ``\textit{the degree to which an observer can understand the cause of a decision}''. Therefore, we can argue that interpretability is as important as accuracy in automatic decisions since it can show if the model can or cannot be trusted. For example, suppose a situation in which a person asks for a credit to his/her bank manager. Moreover, suppose that, after an internal analysis based on a machine learning model, the bank system classifies that person as a possible default and, as a consequence, he/she would not receive the credit. He/she will naturally ask to the bank manager why such a classification was achieved. If the machine is a black box, the manager would not be able to explain such a classification and, therefore, the client may not to trust the algorithm. Therefore, in this situation, a local interpretation would be suitable to understand how each characteristic (e.g., salary, presence or absence of previous default, etc...) contributed to the default credit classification.

There are practically two main types of interpretability in machine learning: global and local ones (see~\citep{Molnar2021} for a further discussion on them). The aim of global interpretability methods consists in explaining the trained model as a whole. In other words, one attempts to derive the average behavior of a trained machine by taking all samples. An example of such a method is the partial dependence plot~\citep{Molnar2021}, whose goal is to provide the marginal effects that each feature has in the predicted outcome. On the other hand, methods for local interpretability attempts to explain, for a specific instance of interest (e.g., a person asking for a credit), how each attribute's value contributes to achieve the associated prediction or classification. In this paper, we deal with local interpretability. Moreover, we consider a model-agnostic approach, i.e., a method that can be applied to interpret the prediction or classification of any machine learning model.

Among the model-agnostic methods proposed in the literature, two are of interest in this paper: LIME (Local Interpretable Model-agnostic Explanations)~\citep{Ribeiro2016} and SHAP (SHapley Additive exPlanations)~\citep{Lundberg2017}. In summary, the idea in LIME to understand the prediction of a specific instance consists in, locally, adjusting an interpretable function (e.g., a linear model) based on a set of perturbed samples in the neighborhood of the instance of interest. When adjusting this linear function, one considers an exponential kernel that ensures that closer are the perturbed samples from the instance of interest, greater are their importance in the learning procedure. Although this function may not be complex enough to explain the model as a whole, it can locally provide a good understanding of the contribution of each attribute in the model prediction. The other approach, called SHAP, brings concepts from game theory to provide local interpretability. The idea is to explain a prediction by means of the Shapley value~\citep{Shapley1953} associated with each attribute value. An interesting aspect in such an approach, which leads to the SHAP values\footnote{In this paper, since the SHAP values are referred to as the Shapley values obtained by means of the SHAP formulation, we will frequently adopt SHAP values or Shapley values interchangeably in the context of machine learning interpretability.}, is that it satisfies desired properties in interpretability, such as local accuracy, missingness and consistency~\citep{Lundberg2017}. For that reason, the classical SHAP and its extended versions have been largely used in the literature~\citep{Lundberg2020,Chen2021,Aas2021}.

Although the Shapley value (as well as the SHAP value in SHAP method) appears as an interesting solution for model-agnostic machine learning interpretability, there is a drawback in its calculation. Since it lies on the marginal contribution of each attribute by taking into account all possible coalitions of attributes, the number of evaluations increases exponentially with the number of attributes. Precisely, if we have $m$ attributes, we need $2^m$ evaluations to calculate the Shapley values. This makes the calculation impracticable in situations where $m$ is large. In order to soften this inconvenience, one may adopt some approaches that approximate the Shapley values, such as the Shapley sampling values strategy~\citep{Strumbelj2010,Strumbelj2014} or the Kernel SHAP~\citep{Lundberg2017}. We address the latter in this paper.

The Kernel SHAP was also proposed in the original SHAP paper~\citep{Lundberg2017}. It provides a link between LIME and the use of SHAP values for local machine learning interpretability. Although the authors provide this link by assuming an additive function as the interpretable model and a specific kernel, the formulation is not straightforward and there is a lack of details in the proof. With respect to the SHAP values calculation, in order to reduce the computational effort, the authors adopted a clever strategy that selects the evaluations that are most promising to approximate such values. However, this strategy does not reduce the number of evaluations needed for an exact SHAP values calculation. Moreover, the authors do not assume further game theory concepts that could speed up the convergence.

Aiming at providing a straightforward formulation of a Kernel SHAP-based method and speeding up the SHAP values approximation, in this paper, we propose to adopt game theory-based concepts frequently used in multicriteria decision making: the Choquet integral~\citep{Choquet1954,Grabisch1996,Grabisch2010} and $k$-additive games~\citep{Grabisch1997b}. Instead of assuming an additive function as the interpretable model, as a first contribution of this paper, we show that the use of the non-additive function called Choquet integral also leads to the same desired properties for local interpretability. Indeed, we can directly associate the Choquet integral parameters to the Shapley indices, which include both Shapley values and Shapley interaction indices. While the Shapley values indicate the marginal contribution of each attribute, individually, the Shapley interaction indices provide the understanding about how they interact between them (positively or negatively). This is of interest in ML interpretability since it indicates if the simultaneous presence of two characteristics has a higher (or lower) contribution than both of them isolated. It is worth mentioning that~\citet{Lundberg2020} also discuss how the Shapley inter SHAP method could be adapted to find the Shapley interaction indices. However, in our proposal, they are obtained automatically.

Besides the aforementioned formulation, we can also assume some degree of additivity about the Choquet integral which contributes to reduce its number of parameters. Therefore, as a second contribution, we propose to adopt a $k$-additive Choquet integral. For instance, $2$-additive models have proved flexible enough to achieve good results in terms of generalization~\citep{Grabisch2002,Grabisch2006,Pelegrina2020}. As attested by numerical experiments, by reducing the number of parameters, we also decrease the number of evaluations needed to approximate the SHAP values.

The rest of this paper is organized as follow. Section~\ref{sec:theor} contains the theoretical aspects of Shapley values and the adopted Choquet integral. We also provide a description of LIME and SHAP as model-agnostic methods for local interpretability. In Section~\ref{sec:prop}, we present our Choquet integral-based formulation that leads to the Shapley interaction indices and how the concept of $k$-additive games can be used to reduce the effort when estimating the SHAP values. Thereafter, in Section~\ref{sec:exp}, we conduct some numerical experiments in order to attest our proposal. Finally, in Section~\ref{concl}, we present our concluding remarks and discuss future perspectives.

\section{Background}
\label{sec:theor}

In this section, we present some theoretical aspects that will be used along this work. We start by some concepts frequently used in game theory and multicriteria decision making. Thereafter, we discuss both LIME and SHAP as well as the SHAP values approximation strategy used in Kernel SHAP.

It is worth recalling that, in this paper, we deal with local interpretability. Therefore, we consider a classical machine learning scenario where a model $f(\cdot)$ (e.g., a black box model) has been trained based on a set of $n_{tr}$ training data $\left(\mathbf{X},\mathbf{y} \right)$, where $\mathbf{X} = \left[\mathbf{x}_1, \ldots, \mathbf{x}_{n_{tr}} \right]$ and $\mathbf{y} = \left[y_1, \ldots, y_{n_{tr}} \right]$ represent the inputs and the output (e.g., a predicted value or a class), respectively. Our aim consists in explaining the predicted outcome $f(\mathbf{x}^*)$ of the instance of interest $\mathbf{x}^* = \left[x_1^*, \ldots, x_m^* \right]$, where $m$ is the number of attributes. Therefore, how $f(\cdot)$ was trained is not important in this paper. We only consider that we are able to use the model $f(\cdot)$ in order to predict the outcome of any instance.

\subsection{Shapley values and $k$-additive games}
\label{subsec:shapley}

In cooperative game theory, a coalitional game is defined by a set $M=\left\{1, 2, \ldots, m\right\}$ of $m$ players and a function $\upsilon: \mathcal{P}(M) \rightarrow \mathbb{R}$, where $\mathcal{P}(M)$ is the power set of $M$, that maps subsets of players to real numbers. For a coalition of players $A$, $\upsilon(A)$ represents the payoff that this coalition can obtain by cooperation. By definition, one assumes $\upsilon(\emptyset) = 0$, i.e., there is no payoff when there is no coalition.

The Shapley value (or Shapley power index) is a well-known solution concept in cooperative game theory~\citep{Shapley1953}. In summary, the Shapley value of a player $j$ indicates its (positive or negative) marginal contribution in the game payoff when taking into account all possible coalitions of players in $M$. It is defined as follows:
\begin{equation}
\label{eq:power_ind_s}
\phi_{j} = \sum_{A \subseteq M\backslash \left\{j\right\}} \frac{\left(m-\left|A\right|-1\right)!\left|A\right|!}{m!} \left[\upsilon(A \cup \left\{j\right\}) - \upsilon(A) \right],
\end{equation}
where $\left| A \right|$ represents the cardinality of subset $A$. An interesting property of the Shapley value (called efficiency, which will be further discussed in this paper) is that $\sum_{j=1}^m \phi_j = \upsilon(M) - \upsilon(\emptyset)$. For this reason, the Shapley value is a convenient way of sharing the payoff of the grand coalition between the players.

Similarly as in Equation~\eqref{eq:power_ind_s}, one may also measure the marginal effect of a coalition $\left\{j,j'\right\}$ in the payoffs. In this case, one obtains the Shapley interaction index, which is defined by~\citep{Murofushi1993,Grabisch1997a}
\begin{equation}
\label{eq:int_pair_ind_s}
I_{j,j'} = \sum_{A \subseteq M\backslash \left\{j,j'\right\}} \frac{\left(m-\left|A\right|-2\right)!\left|A\right|!}{\left(m-1\right)!} \left[\upsilon(A \cup \left\{j,j'\right\}) - \upsilon(A \cup \left\{j\right\}) - \upsilon(A \cup \left\{j'\right\}) + \upsilon(A)\right]
\end{equation}
and can be interpreted as the interaction degree of coalition $\left\{j,j'\right\}$ by taking into account all possible coalitions of players in $M$. The sign of $I_{j,j'}$ indicates the type of interaction between players $j,j'$:
\begin{itemize}

\item If $I_{j,j'} < 0$, there is a negative interaction (also called redundant effect) between players $j,j'$.

\item If $I_{j,j'} > 0$, there is a positive interaction (also called complementary effect) between players $j,j'$.

\item If $I_{j,j'} = 0$, there is no interaction players $j,j'$.

\end{itemize}

Besides $\phi_{j}$ and $I_{j,j'}$, one may also define the interaction index for any $A \subseteq M$. In this case, the (generalized) interaction index is defined by~\citep{Grabisch1997a}
\begin{equation}
\label{eq:int_ind_s}
I(A) = \sum_{D \subseteq M\backslash A} \frac{\left(m-\left|D\right|-\left|A\right|\right)!\left|D\right|!}{\left(m-\left|A\right|+1\right)!} \left( \sum_{D' \subseteq A} \left(-1\right)^{\left| A \right| - \left| D' \right|}\upsilon(D \cup D') \right).
\end{equation}
However, one does not have a clear interpretation as for $\phi_{i}$ and $I_{i,i'}$.

It is important to remark that, given the interaction indices $I(A)$, one may recover the payoffs $\upsilon(A)$ through the linear transformation
\begin{equation}
\label{eq:iitomu_s}
\upsilon(A) = \sum_{D \subseteq M} \gamma^{\left|D \right|}_{\left| A \cap D \right|}I(D),
\end{equation}
where $\gamma^{\left|D \right|}_{\left| A \cap D \right|}$ is defined by
\begin{equation}
\label{eq:gamma}
\gamma^{r'}_{r} = \sum_{l=0}^{r}\binom{r}{l}\eta_{r'-l},
\end{equation}
with
\begin{equation}
\eta_{r} = -\sum_{r'=0}^{r-1}\frac{\eta_{r'}}{r-r'+1}\binom{r}{r'}
\end{equation}
being the Bernoulli numbers and $\eta_0=1$. Since the relation between the game and the interaction indices is linear, it is common to represent the aforementioned transformations using matrix notation. Assume, for instance, that the vectors $\upsilon = [\upsilon(\emptyset), \upsilon(\left\{ 1 \right\}), \ldots, \upsilon(\left\{ m \right\}), \upsilon(\left\{ 1,2 \right\}), \ldots, \upsilon(\left\{ m-1,m \right\}), \ldots,$ $\upsilon(\left\{ 1, \ldots, m \right\}) ]$ and $\mathbf{I} = \left[I(\emptyset), \phi_1, \ldots, \phi_m, I_{1,2}, \ldots, I_{m-1,m}, \ldots, I(\left\{ 1, \ldots, m \right\}) \right]$ are represented in a cardinal-lexicographic order (i.e., the elements are sorted according to their cardinality and, for each cardinality, based on the lexicographic order). The transformation from the interaction indices to $\upsilon$ can be represented by $\upsilon = \mathbf{T} \mathbf{I}$, where $\mathbf{T} \in \mathbb{R}^{2^{m} \times 2^{m}}$ is the transformation matrix. For example, in a game with 3 players, we have
\begin{equation*}
\mathbf{T} =\left[\begin{matrix} 
1 & -1/2 & -1/2 & -1/2 & 1/6 & 1/6 & 1/6 & 0 \\ 
1 & 1/2 & -1/2 & -1/2 & -1/3 & -1/3 & 1/6 & 1/6 \\ 
1 & -1/2 & 1/2 & -1/2 & -1/3 & 1/6 & -1/3 & 1/6 \\ 
1 & -1/2 & -1/2 & 1/2 & 1/6 & -1/3 & -1/3 & 1/6 \\ 
1 & 1/2 & 1/2 & -1/2 & 1/6 & -1/3 & -1/3 & -1/6 \\ 
1 & 1/2 & -1/2 & 1/2 & -1/3 & 1/6 & -1/3 & -1/6 \\ 
1 & -1/2 & 1/2 & 1/2 & -1/3 & -1/3 & 1/6 & -1/6 \\ 
1 & 1/2 & 1/2 & 1/2 & 1/6 & 1/6 & 1/6 & 0 \\ 
\end{matrix}\right].
\end{equation*}

Another concept in game theory directly associated with the interaction indices is the concept of $k$-additive games. We say that a game is $k$-additive if $I(A) = 0$ for all $A$ such that $\left| A \right| > k$. As it will be further detailed in the next section, an advantage of such games is that one reduces the number of parameters to be defined (e.g., from $2^m$ to $m(m+1)/2$ when $k=2$). In the example with 3 players, for instance, if one assumes a $2$-additive game, the last column of $\mathbf{T}$ can be removed since $I({1,2,3}) = 0$.

\subsection{The Choquet integral}
\label{subsec:choquet}


The (discrete) Choquet integral~\citep{Choquet1954} is a non-additive (more precisely, a piecewise linear) aggregation function that models interactions among attributes. It is defined on a set of parameters associated with all possible coalitions of attributes. It has been largely used in multicriteria decision making problems~\citep{Grabisch1996,Grabisch2010} and, in such situations, the parameters associated with the Choquet integral are called capacity coefficients. A capacity is a set function $\mu:2^{M} \rightarrow \mathbb{R}_{+}$ satisfying the axioms of normalization ($\mu(\emptyset) = 0$ and $\mu(M) = 1$) and monotonicity (if $A \subseteq D \subseteq M$, $\mu(A) \leq \mu(D) \leq \mu(M)$). However, the Choquet integral is not restricted to capacities~\citep{Grabisch2016}. Indeed, it can be defined by means of a game $\upsilon:2^{M} \rightarrow \mathbb{R}$ satisfying $\upsilon(\emptyset) = 0$. The Choquet integral definition based on a game $\upsilon$ is given as follows:
\begin{equation}
\label{eq:model_ci}
f_{CI}(\mathbf{x}) = \sum_{j=1}^m (x_{(j)} - x_{(j-1)})\upsilon(\{(j), \ldots, (m)\}),
\end{equation}
where $\cdot_{(j)}$ indicates a permutation of the indices $j$ such that $0 \leq x_{(1)} \leq x_{(j)} \leq \ldots \leq x_{(m)} \leq 1$ (with $x_{(0)}=0$).

Since the Choquet integral is defined by means of a game, one may define it in terms of Shapley values and interaction indices. Therefore, one has a clear interpretation about the marginal contribution of each feature in the aggregation procedure as well as the interaction degree between them. For instance, if two attributes have a positive (resp. negative) interaction, the payoff of such a coalition is (resp. is not) greater than the sum of its individual payoffs. Moreover, one may also consider the case of a $k$-additive game and, therefore, a $k$-additive Choquet integral~\citep{Grabisch1997b}. For example, if one assumes a $2$-additive game,~\eqref{eq:model_ci} can be formulated as follows:
\begin{equation}
\label{eq:choquet_2add}
f_{CI}(\mathbf{x}) = \sum_j x_j \left( \phi_j - \frac{1}{2} \sum_{j'} \left| I_{j,j'} \right| \right) + \sum_{I_{j,j'} < 0} (x_j \vee x_{j'}) \left| I_{j,j'} \right| + \sum_{I_{j,j'} > 0} (x_j \wedge x_{j'}) I_{j,j'},
\end{equation}
where $\vee$ and $\wedge$ represent the maximum and the minimum operators, respectively. Note that, when learning the Choquet integral parameters, if one assumes a $2$-additive model, one reduces the number of parameters from $2^m$ to $m(m+1)/2$. Therefore, $2$-additive and, more generally, $k$-additive models emerge as a strategy that reduces the computational complexity in optimization tasks and provides a more interpretable model (since one has less parameters to interpret). Moreover, it is also known from multicriteria decision making applications~\citep{Grabisch2002,Grabisch2006,Pelegrina2020} that, even if one adopts a $2$-additive model, the Choquet integral is still being flexible enough to model inter-attributes relations and can achieve a high level of generalization.

It is important to remark that, if one assumes that the game is $1$-additive, the Choquet integral becomes a weighted arithmetic mean.

\subsection{Model-agnostic methods for local interpretability}
\label{subsec:lime_shap}

We describe in this section two famous model-agnostic methods for local interpretability: LIME and SHAP. At first, we briefly present the idea behind tabular LIME (i.e., LIME for tabular data). Then, we further discuss the SHAP method, specially the Kernel SHAP strategy. It is worth mentioning that, differently from~\citep{Ribeiro2016,Lundberg2017}, we here adopt a notation based on set theory in order to clearly define the elements used in the considered approaches.

\subsubsection{LIME}
\label{subsubsec:lime}

The main idea of LIME~\citep{Ribeiro2016} for local explanations is to locally approximate a (generally) complex function $f(\cdot)$ (frequently obtained by a black box model) by an interpretable model $g(\cdot)$. For this purpose, in order to explain the outcome $f(\mathbf{x}^*)$ of an instance $\mathbf{x}^*$, one firstly generates a set of $q$ perturbed samples $\mathbf{z}_l$, $l=1, \ldots, q$, in the neighborhood of $\mathbf{x}^*$. For each sample $\mathbf{z}_l$, one also defines a binary vector $\mathbf{z}_l'$ such that $z_{l,j}' = 1$ if $z_{l,j}$ is close enough\footnote{In order to define how close $z_{l,j}$ is from $x_{j}^*$, for each attribute, LIME equally splits the training data (by taking the quantiles of the training data) into predefined bins. Therefore, if $z_{l,j}$ is on the same bin as $x_{j}^*$, $z_{l,j}' = 1$, or $z_{l,j}' = 0$ otherwise. For further details about this procedure, the interested reader may refer to~\citep{Garreau2020}} to $x_{j}^*$, or $z_{l,j}' = 0$ otherwise. Once all samples have been generated, LIME deals with the following optimization problem:
\begin{equation}
    \label{eq:lime_form}
    \min_{g \in G} \mathcal{L}(f,g,\pi_{\mathbf{x}^*}) + \Omega(g),
\end{equation}
where $\mathcal{L}(f,g,\pi_{\mathbf{x}^*})$ is the loss function, $\pi_{\mathbf{x}^*}$ is a proximity measure between the instance to be explained and the perturbed samples and $\Omega(g)$ is a measure of complexity of the interpretable model $g(\cdot)$. In tabular LIME, the authors used the exponential kernel for the proximity measure, which leads to the expression
\begin{equation}
    \pi_{\mathbf{x}^*}(\mathbf{z}_l') = \exp \left(\frac{-\| \mathbbm{1} - \mathbf{z}_l' \|^2}{\alpha^2} \right),
\end{equation}
where $\| \cdot \|$ is the Euclidean norm, $\mathbbm{1}$ is a vector of 1's and $\alpha$ is a positive bandwidth parameter (as default, the authors assumed $\alpha = \sqrt{0.75m}$). By assuming a weighted least squared function for $\mathcal{L}(f,g,\pi_{\mathbf{x}^*})$, a linear function $g(\mathbf{z}') = \beta_0 + \beta^T \mathbf{z}'$ (where $\beta =\left( \beta_1, \ldots, \beta_m \right)$) and letting $\Omega(\beta) = \lambda \| \beta \|^2$ represent a regularization term with $\lambda > 0$, LIME can be formulated as follows:
\begin{equation}
    \label{eq:lime_opt}
    \min_{\beta_0, \beta_1, \ldots, \beta_m} \sum_{l=1}^q \pi_{\mathbf{x}^*}(\mathbf{z}_l') \left( f(\mathbf{z}_l) - \left(\beta_0 + \beta^T \mathbf{z}_l' \right) \right)^2 + \lambda \| \beta \|^2.
\end{equation}
After solving~\eqref{eq:lime_opt}, one can visualize the obtained parameters $\beta$ and, therefore, interpret the (positive or negative) contribution of each attribute in the predicted outcome in the vicinity of $\mathbf{x}^*$.

\subsubsection{SHAP}
\label{subsubsec:shap}

Differently from LIME, the purpose of SHAP is to use the Shapley values in order to locally explain a prediction. The idea is to associate to each attribute its marginal contribution in the predicted outcome. In this section, we present a summary of the idea behind SHAP. Moreover, we discuss the Kernel SHAP, which is a kernel-based approach for approximating the SHAP values by using the LIME formulation. For further details, the interested reader may refer to~\citep{Lundberg2017,Lundberg2018,Lundberg2020,Aas2021}.

The idea that brings Shapley values into interpretability methods in machine learning associates players and payoffs in game theory to attributes and values of a subset of attributes in the model prediction, respectively. Before presenting the idea behind SHAP, let us define the characteristic vector of $A$. Recall that $M = \left\{1, \ldots, m \right\}$ represents the set of $m$ attributes. For any $A \subseteq M$, $\mathbf{1}_A \in \{0,1\}^m$ denotes the characteristic vector of $A$, i.e., a binary vector such that the $j$-th coordinate is 1, if $j \in A$, and 0, otherwise. For example, for $M = \left\{1, 2, 3 \right\}$, $\mathbf{1}_{\left\{2,3\right\}} = \left[0,1,1 \right]$ means a coalition of attributes $\left\{2, 3 \right\}$.

Based on the aforementioned definition, in order to explain the predicted outcome $f(\mathbf{x}^*)$ of an instance $\mathbf{x}^*$, the authors decompose $f(\mathbf{x}^*)$ by assuming the additive feature attribution function given by
\begin{equation}
    \label{eq:shap_g}
    f(\mathbf{x}^*) = g(\mathbf{1}_M) = \phi_0 + \sum_{j \in M} \phi_j.
\end{equation}
Moreover, they argue that the only possible explanation model $g(\cdot)$ that follows Equation~\eqref{eq:shap_g} and satisfies the local accuracy, missingness and consistency properties (see Appendix A for the definitions) consists in defining $\phi_0 = \mathbb{E}\left[f(\mathbf{x})\right]$, i.e., the (overall) expected prediction when one does not know any attribute value from $\mathbf{x}^*$, and the (exact) SHAP values $\phi_{j}$, $j=1, \ldots, m$, given by
\begin{equation}
\label{eq:shap_values}
\phi_{j}(f,\mathbf{x}^*) = \sum_{A \subseteq M \backslash \left\{j\right\}} \frac{\left(m-\left| A \right|-1\right)! \left| A \right|!}{m!} \left[ \hat{f}_{\mathbf{x}^*}( A \cup \left\{j\right\}) - \hat{f}_{\mathbf{x}^*}( A) \right],
\end{equation}
where $\hat{f}_{\mathbf{x}^*}( A)$ is the expected model prediction given the knowledge on the attributes values of $\mathbf{x}^*$ that are present in coalition $A$, that is:
\begin{equation}
\label{eq:exp_pred}
\hat{f}_{\mathbf{x}^*}( A) = \mathbb{E}\left[f\left( \mathbf{x} \right) | x_j = x_j^* \, \, \forall \, \, j \in A \right].
\end{equation}
Note in Equation~\eqref{eq:exp_pred} that one has missing values for all attributes $j' \in \overline{A}$, where $\overline{A}$ is the complement set of $A$ (if $A = M$, then $\hat{f}_{\mathbf{x}^*}( M) = \mathbb{E}\left[f\left( \mathbf{x}^* \right) \right] = f\left( \mathbf{x}^* \right)$ and there are no missing values). In this case, in order to calculate the expected prediction, one randomly samples these missing values from the training data. In this paper, as well as in the Kernel SHAP method, we assume independence among attributes. Therefore, the expected prediction can be calculated as follows:
\begin{equation}
\label{eq:exp_pred_indep}
\hat{f}_{\mathbf{x}^*}( A) = \frac{1}{q} \sum_{l=1}^q f\left(\mathbf{x}_{A}^*,\mathbf{x}_{l,\overline{A}} \right),
\end{equation}
where $\mathbf{x}_{l,\overline{A}}$, $l=1, \ldots, q$, are samples from the training data. Note that, in comparison with the game theory formulation presented in Equation~\eqref{eq:power_ind_s}, $\hat{f}_{\mathbf{x}^*}( A)$ represents the payoff $\upsilon(A)$. Moreover, when all attributes are missing, i.e., $A = \emptyset$, one has $\hat{f}_{\mathbf{x}^*}( \emptyset ) = \mathbb{E}\left[f\left( \mathbf{x} \right)\right] = \phi_0$. 

Among the properties satisfied by the SHAP values, the local accuracy plays an important role in local interpretability and differentiate SHAP from the original LIME formulation (as presented in Section~\ref{subsubsec:lime}). It states that one can decompose the predicted outcome $f(\mathbf{x}^*)$ by the sum of the SHAP values and the overall expected prediction $\phi_0$, i.e., $f(\mathbf{x}^*) = \phi_0 + \sum_{j=1}^m \phi_j$. Therefore, one may interpret the SHAP values as the contribution of each attribute when one moves from the overall expected prediction when all attributes are missing to the actual outcome $f(\mathbf{x}^*)$.

\subsubsection{Kernel SHAP}
\label{subsubsec:kernel_shap}

An important remark in the exact SHAP values calculation is that one needs to sample all $2^m$ possible coalitions of attributes and calculate its expected model prediction. Therefore, this procedure may be computationally heavy for a large number of attributes. In order to overcome this inconvenience, the authors proposed a SHAP value-based formulation called Kernel SHAP~\citep{Lundberg2017}. Kernel SHAP emerges as the formulation of LIME method that leads to the SHAP values. For instance, the authors claimed that if one assumes
\begin{itemize}
    \item $\Omega(g) = 0$,
    \item $\pi(A) = \frac{(m-1)}{\binom{m}{\left|A\right|} \left|A\right| (m - \left|A\right|)}$,
    \item $\mathcal{L}(f,g,\pi) = \sum_{A \in \mathcal{M}} \pi(A) \left( \hat{f}_{\mathbf{x}^*}( A) - g(\mathbf{1}_A) \right)^2$,
    where $g(\mathbf{1}_A) = \phi_0 + \sum_{j \in A} \phi_j$ and $\mathcal{M} \subseteq \mathcal{P}(M)$ (recall that $\mathcal{P}(M)$ is the power set of $M$),
\end{itemize}
the solution of the weighted least square problem
\begin{equation}
    \label{eq:shap_opt}
    \min_{\phi_0, \phi_1, \ldots, \phi_m} \sum_{A \in \mathcal{M}} \frac{(m-1)}{\binom{m}{\left|A\right|} \left|A\right| (m - \left|A\right|)} \left( \hat{f}_{\mathbf{x}^*}( A) - \left( \phi_0 + \sum_{j \in A} \phi_j \right)\right)^2
\end{equation}
leads to the SHAP values. Note that, differently from the LIME formulation, $\pi(A)$ in Kernel SHAP only depends on coalition $A$. Moreover, $\pi(A)$ tends to infinity when $A = M$. Therefore, in the optimal solution, $\hat{f}_{\mathbf{x}^*}( M) = f(\mathbf{x}^*) = g(\mathbf{1}_M) = \phi_0 + \sum_{j=1}^m \phi_j$. This ensures that $f(\mathbf{x}^*)$ is explained by the sum of the SHAP values and the overall expected prediction $\mathbb{E}\left[f(\mathbf{x})\right]$. Similarly, when $A = \emptyset$, the associated $\pi(\emptyset)$ also tends to infinity. This ensures that $\hat{f}_{\mathbf{x}^*}( \emptyset ) = \mathbb{E}\left[f(\mathbf{x})\right] = g(\mathbf{1}_{\emptyset}) = \phi_0$. In practice, we replace these infinite values by a big constant (e.g., $10^6$).

As~\eqref{eq:shap_opt} is a weighted least square problem, one may easily represent it (as well as its solution) by means of matrices and vectors (we borrow such a formulation from~\citep{Aas2021}). Suppose that $n_{\mathcal{M}}$ represents the number of elements in $\mathcal{M}$ (i.e., the number of coalitions considered in the optimization problem~\eqref{eq:shap_opt}). Let us also define $\phi = \left[ \phi_0, \phi_1, \ldots, \phi_m \right]$ and $\mathbf{Z} \in \left\{0,1\right\}^{n_{\mathcal{M}} \times (m+1)}$ as the matrix such that the first column is 1 for every row and the remaining $m+1$ columns are composed, in each row, by all $\mathbf{1}_A$, $A \in \mathcal{M}$. Moreover, assume that $\mathbf{f} \in \mathbb{R}^{n_{\mathcal{M}} \times 1}$ and $\mathbf{W} \in \mathbb{R}^{n_{\mathcal{M}} \times n_{\mathcal{M}}}$ are the vector of evaluations $\hat{f}_{\mathbf{x}^*}( A)$ and the diagonal matrix whose elements are given by $\pi(A)$, respectively, associated with all $A \in \mathcal{M}$. For example, in a problem with 3 attributes ($M=\left\{1,2,3 \right\}$) and using $\emptyset$, $\left\{1 \right\}$, $\left\{2 \right\}$, $\left\{1,3 \right\}$ and $M$ as the coalitions of attributes, we have the following:
\begin{equation*}
\phi = \left[\begin{matrix} 
\phi_0 \\
\phi_1 \\
\phi_2 \\
\phi_3
\end{matrix}\right], \, \, \mathbf{Z} = \left[\begin{matrix} 
1 & 0 & 0 & 0 \\
1 & 1 & 0 & 0 \\
1 & 0 & 1 & 0 \\
1 & 1 & 0 & 1 \\
1 & 1 & 1 & 1 \\
\end{matrix}\right], \, \, \mathbf{f} = \left[\begin{matrix} 
\hat{f}_{\mathbf{x}^*}( \emptyset) \\
\hat{f}_{\mathbf{x}^*}( \left\{1 \right\} ) \\
\hat{f}_{\mathbf{x}^*}( \left\{2 \right\} ) \\
\hat{f}_{\mathbf{x}^*}( \left\{1,3 \right\} ) \\
\hat{f}_{\mathbf{x}^*}( M) \\
\end{matrix}\right]\text{ and }\mathbf{W} = \left[\begin{matrix} 
10^6 & 0 & 0 & 0 & 0 \\
0 & \pi(\left\{1 \right\}) & 0 & 0 & 0 \\
0 & 0 & \pi(\left\{2 \right\}) & 0 & 0 \\
0 & 0 & 0 & \pi(\left\{1,3 \right\}) & 0 \\
0 & 0 & 0 & 0 & 10^6 \\
\end{matrix}\right].
\end{equation*}
Based on the vector/matrix notation, one may represent the optimization problem~\eqref{eq:shap_opt} as
\begin{equation}
    \label{eq:shap_opt_matrix}
    \min_{\phi} \left(\mathbf{f} - \mathbf{Z}\phi \right)^T \mathbf{W} \left(\mathbf{f} - \mathbf{Z}\phi \right),
\end{equation}
whose solution is given by
\begin{equation}
    \label{eq:shap_opt_matrix_sol}
    \phi = \left(\mathbf{Z}^T\mathbf{W}\mathbf{Z} \right)^{-1} \mathbf{Z}^T\mathbf{W}\mathbf{f}.
\end{equation}

Remark that $\mathbf{S} = \left(\mathbf{Z}^T\mathbf{W}\mathbf{Z} \right)^{-1} \mathbf{Z}^T\mathbf{W}$ can be calculated independently of the instance of interest $\mathbf{x}^*$. Therefore, an interesting aspect in Kernel SHAP is that, even if one would like to explain the outcome of several instances of interest, one only needs to calculate $\mathbf{S}$ once. The only element that varies in Equation~\eqref{eq:shap_opt_matrix_sol} is the vector of evaluations $\mathbf{f}$, which is dependent on the instance of interest under analysis.

Another remark in Kernel SHAP is that, if $\mathcal{M} = \mathcal{P}(M)$, Equation~\eqref{eq:shap_opt_matrix_sol} leads to the exact SHAP values (as in Equation~\eqref{eq:shap_values}). Therefore, in this exact calculation, one needs the expected predictions $\hat{f}_{\mathbf{x}^*}( A )$ for all possible coalitions $A$, which can be infeasible for a large number of attributes. However, the clever strategy used in Kernel SHAP aims at selecting the most promising expected predictions to approximate the SHAP values. For instance, if one considers the weighting kernel $\pi(A)$, one may note that the majority of $A$ has a low contribution in the SHAP value calculation. Therefore, the aim in Kernel SHAP consists in defining a subset $\mathcal{M}$ from $\mathcal{P}(M)$ such that the elements $A \in \mathcal{M}$ are sampled\footnote{In order to avoid double selecting the same $A$, in the experiments conducted in this paper, we adopted a sampling procedure without replacement. Therefore, after sampling a coalition, we update the probability distribution by removing the associated kernel weight and normalizing the probabilities.} from a probability distribution following the weighting kernel $\pi(A)$. Greater is the weight associated with $A$, greater is the chance that $A$ is sampled from $\mathcal{P}(M)$.

\section{A more general model for local interpretability based on Shapley values}
\label{sec:prop}

As highlighted in Section~\ref{sec:intro}, we have two main contributions in this paper: to provide a straightforward formulation of the Kernel SHAP method based on the Choquet integral, and to adopt the concept of $k$-additive games in order to reduce the number of evaluations needed to approximate the SHAP values. Both contributions are presented in the sequel.

\subsection{The Choquet integral as an interpretable model for Kernel SHAP formulation}

We here show that we need not consider an additive function as the interpretable model in order to explain a prediction based on the Shapley values. Indeed, if we adopt the non-additive function called Choquet integral, we also achieve such values. Recall the Choquet integral function defined in Equation~\eqref{eq:model_ci}. The idea is to define the local interpretable model $g(\cdot)$ as
\begin{equation}
    \label{eq:choquet_lime}
    g(\mathbf{1}_A) = \phi_0 + f_{CI}(\mathbf{1}_A),
\end{equation}
where $\phi_0$ is the intercept parameter. In order to simplify the notation, let us also define $\bar{f}_{\mathbf{x}^*}( A) = \hat{f}_{\mathbf{x}^*}( A) - \phi_0$. In this case and based on the LIME formulation for local interpretability, one obtains the following loss function:
\begin{equation}
    \label{eq:choquet_lime_loss}
    \mathcal{L}(f,g,\pi) = \sum_{A \in \mathcal{M}}\pi'(A) \left( \bar{f}_{\mathbf{x}^*}( A) - f_{CI}(\mathbf{1}_A) \right)^2,
\end{equation}
where the weights $\pi'(A)$ have the same values for all $A$ (e.g., 1) except for the empty set and the grand coalition $M$, whose associated weights are big numbers (e.g., $10^6$). We clarify these choices soon.

An interesting aspect on the Choquet integral and that can be easily checked from Equation~\eqref{eq:model_ci} is that, when we only have binary data (which is our case since $\mathbf{1}_A$ is a binary vector), $f_{CI}(\mathbf{1}_A) = \upsilon(A)$. Therefore, we may redefine the loss function presented in~\eqref{eq:choquet_lime_loss} as
\begin{equation}
    \label{eq:choquet_lime_loss_game}
    \mathcal{L}(f,g,\pi) = \sum_{A \in \mathcal{M}} \pi'(A) \left( \bar{f}_{\mathbf{x}^*}( A) - \upsilon(A) \right)^2.
\end{equation}
Remark that, for $A = \emptyset$, we minimize $\bar{f}_{\mathbf{x}^*}( \emptyset) - \upsilon(\emptyset) = \hat{f}_{\mathbf{x}^*}( \emptyset) - \phi_0 - \upsilon(\emptyset) = \upsilon(\emptyset)$, since $\hat{f}_{\mathbf{x}^*}( \emptyset) = \phi_0$ by definition. In order to ensure that $\upsilon(\emptyset) = 0$ (according to the definition of a game), we assume a big number for $\pi'(\emptyset)$ when solving the optimization problem. Similarly, when $A=M$, we minimize the difference between $\hat{f}_{\mathbf{x}^*}( M)$ and $\phi_0 + \upsilon(M)$. In this case, since $\upsilon(M) = \upsilon(\emptyset) + \sum_{j=1}^m \phi_j$, the big weight $\pi'(M)$ ensures that $\phi_0 + \sum_{j=1}^m \phi_j = \hat{f}_{\mathbf{x}^*}( M) = f(\mathbf{x}^*)$ (the local accuracy property).

Furthermore, if one considers the linear transformation presented in Equation~\eqref{eq:iitomu_s}, the loss function can be directly defined in terms of the generalized Shapley interaction indices. In this case, we have the following optimization problem:
\begin{equation}
    \label{eq:shap_opt_ci}
    \min_{\mathbf{I}} \sum_{A \in \mathcal{M}} \pi'(A) \left( \bar{f}_{\mathbf{x}^*}( A) - \sum_{D \subseteq M} \gamma^{\left|D \right|}_{\left| A \cap D \right|}I(D) \right)^2,
\end{equation}
where $\gamma$ is defined as in Equation~\eqref{eq:gamma}. As in the Kernel SHAP, our proposal also leads to the exact SHAP values if $\mathcal{M}=\mathcal{P}(M)$. We prove it in the sequel.

\begin{theor}{}
\label{theo1}
If $\mathcal{M} = \mathcal{P}(M)$, the solution of~\eqref{eq:shap_opt_ci} leads to the exact SHAP values as calculated in Equation~\eqref{eq:shap_values}.
\end{theor}

\begin{proof}
Assume $\mathcal{M} = \mathcal{P}(M)$. In this scenario, the optimization problem~\eqref{eq:shap_opt_ci} has a unique solution such that $\sum_{D \subseteq M} \gamma^{\left|D \right|}_{\left| A \cap D \right|}I(D) = \upsilon(A) = \bar{f}_{\mathbf{x}^*}( A)$. From the obtained game and the linear transformation presented in Equation~\eqref{eq:power_ind_s}, we have that $\phi_{j} = \sum_{A \subseteq M\backslash \left\{j\right\}} \frac{\left(m-\left|A\right|-1\right)!\left|A\right|!}{m!} \left[\upsilon(A \cup \left\{j\right\}) - \upsilon(A) \right]$. It remains to show that $\phi_{j} \equiv \phi_{j}(f,\mathbf{x}^*)$.

Recall that we defined $\bar{f}_{\mathbf{x}^*}( A) = \hat{f}_{\mathbf{x}^*}( A) - \phi_0$ and, then, $\upsilon(A) = \hat{f}_{\mathbf{x}^*}( A) - \phi_0$ in the optimal solution. Therefore, we have the following:
\begin{equation}
\label{eq:proof_equiv}
    \begin{split}
        \phi_{j} & = \sum_{A \subseteq M\backslash \left\{j\right\}} \frac{\left(m-\left|A\right|-1\right)!\left|A\right|!}{m!} \left[\hat{f}_{\mathbf{x}^*}( A \cup \left\{j\right\}) - \phi_0 - \hat{f}_{\mathbf{x}^*}( A) + \phi_0 \right] \\
        & = \sum_{A \subseteq M\backslash \left\{j\right\}} \frac{\left(m-\left|A\right|-1\right)!\left|A\right|!}{m!} \left[\hat{f}_{\mathbf{x}^*}( A \cup \left\{j\right\}) - \hat{f}_{\mathbf{x}^*}( A) \right] \\
        & = \phi_{j}(f,\mathbf{x}^*),
    \end{split}
\end{equation}
which proves that our proposal also converges to the exact SHAP values when $\mathcal{M} = \mathcal{P}(M)$.
\end{proof}

Similarly as in the Kernel SHAP, we may here also rewrite the optimization problem in vector/matrix notation. For this purpose, let us represent $\hat{\mathbf{f}} \in \mathbb{R}^{n_{\mathcal{M}} \times 1}$ as the vector $\mathbf{f}$ (as defined in Section~\ref{subsubsec:kernel_shap}) discounted by $\phi_0$ and $\bar{\mathbf{W}} \in \mathbb{R}^{n_{\mathcal{M}} \times n_{\mathcal{M}}}$ as the diagonal matrix whose elements are 1's except for the elements associated with the empty set and the grand coalition $M$, whose weights are a big number (e.g., $10^6$). Moreover, we define $\upsilon_{\mathcal{M}}$ as the vector of payoffs for all coalitions $A$ such that $A \in \mathcal{M}$. In addition, we consider $\mathbf{T}_{\mathcal{M}}$ as the transformation matrix whose rows are composed by the rows of $\mathbf{T}$ (as defined in Section~\ref{subsec:shapley}) associated with all coalitions $A$ such that $A \in \mathcal{M}$. For example, in the same problem when $M=\left\{1,2,3 \right\}$ and using $\emptyset$, $\left\{1 \right\}$, $\left\{2 \right\}$, $\left\{1,3 \right\}$ and $M$ as the coalitions of attributes, we have the following:
\begin{equation*}
\upsilon_{\mathcal{M}} = \left[\begin{matrix} 
\upsilon(\emptyset) \\
\upsilon(\left\{1 \right\}) \\
\upsilon(\left\{2 \right\}) \\
\upsilon(\left\{1,3 \right\}) \\
\upsilon(M)
\end{matrix}\right] = \left[\begin{matrix} 
1 & -1/2 & -1/2 & -1/2 & 1/6 & 1/6 & 1/6 & 0 \\ 
1 & 1/2 & -1/2 & -1/2 & -1/3 & -1/3 & 1/6 & 1/6 \\ 
1 & -1/2 & 1/2 & -1/2 & -1/3 & 1/6 & -1/3 & 1/6 \\ 
1 & 1/2 & -1/2 & 1/2 & -1/3 & 1/6 & -1/3 & -1/6 \\ 
1 & 1/2 & 1/2 & 1/2 & 1/6 & 1/6 & 1/6 & 0
\end{matrix}\right] \left[\begin{matrix} 
I(\emptyset) \\
\phi_1 \\
\phi_2 \\
\phi_3 \\
I_{1,2} \\
I_{1,3} \\
I_{2,3} \\
I(\left\{1,2,3 \right\})
\end{matrix}\right] = \mathbf{T}_{\mathcal{M}} \mathbf{I}.
\end{equation*}
\begin{equation*}
\hat{\mathbf{f}} = \left[\begin{matrix} 
\hat{f}_{\mathbf{x}^*}( \emptyset ) - \phi_0 \\
\hat{f}_{\mathbf{x}^*}( \left\{1 \right\} ) - \phi_0 \\
\hat{f}_{\mathbf{x}^*}( \left\{2 \right\} ) - \phi_0 \\
\hat{f}_{\mathbf{x}^*}( \left\{1,3 \right\} ) - \phi_0 \\
\hat{f}_{\mathbf{x}^*}( M ) - \phi_0 \\
\end{matrix}\right]\text{ and }\bar{\mathbf{W}} = \left[\begin{matrix} 
10^6 & 0 & 0 & 0 & 0 \\
0 & 1 & 0 & 0 & 0 \\
0 & 0 & 1 & 0 & 0 \\
0 & 0 & 0 & 1 & 0 \\
0 & 0 & 0 & 0 & 10^6 \\
\end{matrix}\right].
\end{equation*}
The vector/matrix notation leads to the following optimization problem:
\begin{equation}
    \label{eq:shap_opt_matrix_ci}
    \min_{\mathbf{I}} \left(\bar{\mathbf{f}} - \mathbf{T}_{\mathcal{M}} \mathbf{I} \right)^T \bar{\mathbf{W}} \left(\bar{\mathbf{f}} - \mathbf{T}_{\mathcal{M}} \mathbf{I} \right),
\end{equation}
whose solution is given by
\begin{equation}
    \label{eq:shap_opt_matrix_sol_ci}
    \mathbf{I} = \left(\mathbf{T}_{\mathcal{M}}^T\bar{\mathbf{W}}\mathbf{T}_{\mathcal{M}} \right)^{-1} \mathbf{T}_{\mathcal{M}}^T\bar{\mathbf{W}}\bar{\mathbf{f}}.
\end{equation}

It is important to note that, differently from the Kernel SHAP formulation discussed in Section~\ref{subsubsec:kernel_shap}, in our proposal we obtain all Shapley interaction indices (which, obviously, include the SHAP values). Therefore, this can also be infeasible for a large number of attributes, since the number of parameters is given by $2^m$. However, one may exploit some degree of additivity about the Choquet integral which contributes to reduce its number of parameters. We discuss this aspect in the next section.

\subsection{$k$-additive games for local interpretability}
\label{subsec:kadd}

As a second contribution, we propose to adopt the concept of $k$-additive games in the Choquet integral-based formulation for local interpretability. Called here $k_{ADD}$-SHAP, our proposal consists in dealing with the following weighted least square problem:
\begin{equation}
    \label{eq:shap_opt_ci_kadd}
    \min_{\phi_0,\mathbf{I}_k} \sum_{A \in \mathcal{M}} \pi'(A) \left( \bar{f}_{\mathbf{x}^*}( A) - \sum_{\substack{D \subseteq M, \\ \left| D \right| \leq k}} \gamma^{\left|D \right|}_{\left| A \cap D \right|}I(D) \right)^2,
\end{equation}
where $\mathbf{I}_k = \left[I(\emptyset), \phi_1, \ldots, \phi_m, I_{1,2}, \ldots, I(\left\{m-k, \ldots, m\right\} \right]$ is the vector of Shapley interaction indices, in a cardinal-lexicographic order, for all $I(D)$ such that $\left| D \right| \leq k$. By using the vector/matrix notation, we may rewrite~\eqref{eq:shap_opt_ci_kadd} as follows:
\begin{equation}
    \label{eq:shap_opt_matrix_ci_kadd}
    \min_{\mathbf{I}_k} \left(\bar{\mathbf{f}} - \mathbf{T}_{\mathcal{M},k} \mathbf{I}_k \right)^T \bar{\mathbf{W}} \left(\bar{\mathbf{f}} -  \mathbf{T}_{\mathcal{M},k} \mathbf{I}_k \right),
\end{equation}
whose solution is given by
\begin{equation}
    \label{eq:shap_opt_matrix_sol_ci_kadd}
    \mathbf{I}_k = \left(\mathbf{T}_{\mathcal{M},k}^T\bar{\mathbf{W}}\mathbf{T}_{\mathcal{M},k} \right)^{-1} \mathbf{T}_{\mathcal{M},k}^T\bar{\mathbf{W}}\bar{\mathbf{f}},
\end{equation}
where $\mathbf{T}_{\mathcal{M},k}$ is equal to $\mathbf{T}_{\mathcal{M}}$ up to the columns associated with all $I(D')$ such that $\left| D' \right| \leq k$ ($I(D') = 0$ for all coalitions $D'$ such that $\left| D' \right| > k$). 

Note that, as in~\eqref{eq:shap_opt_ci_kadd} (or\eqref{eq:shap_opt_matrix_ci_kadd}) we restrict the feasible domain to Shapley indices whose cardinalities are at most $k$, we can not guarantee to achieve the exact SHAP values even if $\mathcal{M} = \mathcal{P}(M)$. In other words, Theorem~\ref{theo1} is not valid for the proposed $k_{ADD}$-SHAP. However, as already mentioned in Section~\ref{subsec:choquet}, an advantage of such a model is that one both drastically reduces the number of parameters to be determined while still having a flexible model to generalize the relation between inputs and outputs. Therefore, in order to approximate the exact SHAP values, we avoid over-parametrization and we may need less evaluations when adopting~\eqref{eq:shap_opt_matrix_ci_kadd} compared to~\eqref{eq:shap_opt_matrix}. It is important to note that, even if in the Kernel SHAP formulation one only searches for the Shapley values, implicitly, one needs the expected predicted evaluations on all coalitions of attributes in order to calculate the exact SHAP values.

With respect to how to select the subset $\mathcal{M}$ of evaluations, we consider the same strategy as in Kernel SHAP. We sample the elements $A \in \mathcal{M}$ according to the probability distribution defined by $p_A = \frac{\pi(A)}{\sum_{A \subseteq M} \pi(A)}$. As we adopt in this paper a sampling procedure without replacement, after sampling a coalition, we update the probability distribution and normalize it. Moreover, as $p_\emptyset$ and $p_M$ are much greater than the other probabilities, it is very likely that both the empty set and the grand coalition $M$ are sampled to compose the subset $\mathcal{M}$.

Equivalently as in the Kernel SHAP formulation, $\mathbf{S}_{\mathcal{M},k} = \left(\mathbf{T}_{\mathcal{M},k}^T\bar{\mathbf{W}}\mathbf{T}_{\mathcal{M},k} \right)^{-1} \mathbf{T}_{\mathcal{M},k}^T\bar{\mathbf{W}}$ (or $\mathbf{S}_{\mathcal{M}} = \left(\mathbf{T}_{\mathcal{M}}^T\bar{\mathbf{W}}\mathbf{T}_{\mathcal{M}} \right)^{-1} \mathbf{T}_{\mathcal{M}}^T\bar{\mathbf{W}}$) can also be calculated independently of the instance of interest $\mathbf{x}^*$. Therefore, in order to explain the outcome of several instances of interest, one only needs to calculate $\mathbf{S}_{\mathcal{M},k}$ (or $\mathbf{S}_{\mathcal{M}}$) once.

\section{Numerical experiments}
\label{sec:exp}

In this section, we present some numerical experiments in order to check the validity and interest of our model\footnote{All codes can be accessed in \url{https://github.com/GuilhermePelegrina/k_addSHAP}.}. The experiments are based on two datasets frequently used in the literature: Diabetes~\citep{Efron2004} and Red Wine Quality~\citep{Cortez2009}. In the sequel, we provide a brief description of both datasets:
\begin{itemize}
\item Diabetes dataset: This dataset contains $m=10$ attributes (age, sex, body mass index, average blood pressure and six blood serum measurements) that describe $n=442$ diabetes patients. All collected data are centralized (with zero mean) and with standard deviation equal to $0.0476$. For each patient, one also has as the predicted value a measure of the diabetes progression. The mean and the standard deviation for the diabetes progression measure are $152.13$ and $77.00$, respectively. In our experiments, we split the dataset into training (80\%, i.e., $n_{tr}=353$ samples) and test (20\%, i.e., $n_{te}=89$ samples).

\item Red Wine Quality dataset: In this dataset, one has $m=11$ attributes describing $n=1599$ red wines. Both mean and standard deviation (std) of attributes are described in Table~\ref{tab:wine}. For each wine, one also has a score (between 0 and 10) indicating its quality. In our experiments, we use this data for the purpose of classification and, therefore, we assume that a good (resp. a bad) wine has a score greater than 5 (resp. at most 5). In total, one has 855 good wine (class value 1) and 744 bad wine (class value 0). Moreover, we split the dataset into training (80\%, i.e., $n_{tr}=1279$ samples) and test (20\%, i.e., $n_{te}=320$ samples).
\end{itemize}

\begin{table}[ht]
  \begin{center}
  \caption{Summary of the Wine dataset.}\label{tab:wine}
  {
  \renewcommand{\arraystretch}{1.0}
	\small
  \begin{tabular}{cccccccc}
	\cline{1-2}\cline{4-5}\cline{7-8}
\textbf{Attributes} & \makecell{ \textbf{Mean} \\ \textbf{($\pm$ std)} } &  & \textbf{Attributes} & \makecell{ \textbf{Mean} \\ \textbf{($\pm$ std)} } &  & \textbf{Attributes} & \makecell{ \textbf{Mean} \\ \textbf{($\pm$ std)} } \\
		\cline{1-2}\cline{4-5}\cline{7-8}
		fixed acidity  & \makecell{ $8.320$ \\ ($\pm 1.740$) } &  & chlorides & \makecell{ $0.087$ \\ ($\pm 0.047$) } &  & pH & \makecell{ $3.311$ \\ ($\pm 0.154$) } \\
		\cline{1-2}\cline{4-5}\cline{7-8}
		volatile acidity  & \makecell{ $0.528$ \\ ($\pm 0.179$) } &  & free sulfur dioxide & \makecell{ $15.875$\\ ($\pm 10.457$) } &  & sulphates & \makecell{ $0.658$ \\ ($\pm 0.169$) } \\
		 \cline{1-2}\cline{4-5}\cline{7-8}
		citric acid & \makecell{ $0.271$ \\ ($\pm 0.195$) } &  & total sulfur dioxide & \makecell{ $46.468$ \\ ($\pm 32.885$) } &  & alcohol & \makecell{ $10.423$ \\ ($\pm 1.065$) } \\
		\cline{1-2}\cline{4-5}\cline{7-8}
		residual sugar & \makecell{ $2.539$ \\ ($\pm 1.409$) } &  & density & \makecell{ $0.997$ \\ ($\pm 0.002$) } &  &  &  \\
		\cline{1-2}\cline{4-5}\cline{7-8}
  \end{tabular}
  }
  \end{center}
\end{table}

Besides different datasets, we also evaluate our proposal by assuming two training models: Neural Network and Random Forest\footnote{We borrowed these methods from the Scikit-learn library~\citep{Pedregosa2011} in Python and adopted the following parameters:
\begin{itemize}
    \item Neural Network: $max_{iter} = 10^6$ for both MLPRegressor and MLPClassifier.
    \item Random Forest: $n\_estimators=1000$, $max\_depth=None$ and $min\_samples\_split=2$ for both RandomForestRegressor and RandomForestClassifier.
\end{itemize}}. Recall that the purpose of this paper is to address interpretability in any trained machine learning models. We do not work on improving the model itself. So we attempt to explain the contributions of attributes regardless how the model is accurate.

\subsection{Experiment varying the number of expected prediction evaluations until the exact SHAP values convergence}
\label{subsec:exp1}

In the first experiment, we verify the convergence of the proposed $k_{ADD}$-SHAP and the Kernel SHAP to the exact SHAP values. For each dataset and test sample (recall that we use the training data to calculate the expected predictions given the coalitions in $\mathcal{M}$), we vary the number of expected prediction evaluations, apply both $k_{ADD}$-SHAP and Kernel SHAP and calculate the squared error when estimating the exact SHAP values. Let us represent, for a given test sample $i'$, the SHAP values obtained by the Equation~\eqref{eq:shap_values} (the exact SHAP values), the $k_{ADD}$-SHAP and the Kernel SHAP as $\phi^{exact,i'}$, $\phi^{k_{ADD},i'}$ and $\phi^{Kernel,i'}$, respectively. The squared error between $\phi^{exact,i'}$ and $\phi^{k_{ADD},i'}$ is given as follows:
\begin{equation}
\label{eq:error}
    \varepsilon_{k_{ADD},i'} = \sum_{j=1}^m \left(\phi_j^{exact,i'} - \phi_j^{k_{ADD},i'} \right)^2.
\end{equation}
In order to calculate the squared error with respect to the Kernel SHAP, one only needs to replace $\phi_j^{k_{ADD},i'}$ by $\phi^{Kernel,i'}$. By increasing $n_{\mathcal{M}}$, i.e., the number of coalitions selected to calculate the expected prediction evaluations used in the estimation procedure, the aim is to verify the convergence to the exact SHAP values. We show the obtained results by taking the median (50th percentile or $q_{0.5}$ - a central tendency measure), the 90th percentile and the 10th percentile ($q_{0.9}$ and $q_{0.5}$, respectively, both used to indicate the dispersion around the median) over $s = 501$ simulations. For each simulation, we calculate the errors when estimating the exact SHAP values of all test samples. For the proposed $k_{ADD}$-SHAP, the percentile $q_a$, $a=0.1,0.5,0.9$, is calculated as follows:
\begin{equation}
\label{eq:average_error}
    \bar{\varepsilon}_{a,k_{ADD}} = q_a \left( \frac{1}{n_{te}}\sum_{i'=1}^{n_{te}} \varepsilon_{k_{ADD},i'}^{1}, \ldots, \frac{1}{n_{te}}\sum_{i'=1}^{n_{te}} \varepsilon_{k_{ADD},i'}^{s} \right)
\end{equation}
where $\varepsilon_{k_{ADD},i'}^{r}$, $r=1, \ldots, s$ represents the squared error for test sample $i'$ in simulation $r$. Equation~\eqref{eq:average_error} can be easily adapted to calculate the metrics when adopting the Kernel SHAP.

The results are presented in Figures~\ref{fig:convergence_diabetes} and~\ref{fig:convergence_wine}. The central line represents the average median and the shaded area indicates the averaged dispersion between the 10th and 90th percentiles. For both datasets and trained models, the $3_{ADD}$-SHAP leads to a faster approximation to the exact SHAP values in comparison with the Kernel SHAP. Moreover, the dispersion was lower for the $3_{ADD}$-SHAP even with reduced numbers of expected prediction evaluations. For Kernel SHAP, one achieves a high dispersion for low number of expected prediction evaluations (see, especially, Figure~\ref{fig:convergence_diabetes}), which decreases as one includes more samples. With respect to the $2_{ADD}$-SHAP, it has a good performance (better than the Kernel SHAP) for few evaluations, however, it diverges as more samples are include in the SHAP values estimation. An explanation for these results is that the $2_{ADD}$-SHAP could rapidly approximate the exact SHAP values when less evaluations are used because it can avoid over-parametrization when only few data are considered. However, when increasing the number of expected prediction evaluations, the $2_{ADD}$-SHAP has not enough flexibility to model the data and the adjusted parameters could not converge to the correct ones. As can be also seen in Figure~\ref{fig:convergence_wine}, the $3_{ADD}$-SHAP also diverges for a high number of evaluations (recall from Section~\ref{subsec:kadd} that we can not guarantee to achieve the exact SHAP values even if $\mathcal{M} = \mathcal{P}(M)$), however, it still achieves a very low error. Clearly, as we increase $k$, the parameters become more flexible to model the data and estimate the exact SHAP values.

\begin{figure}[!h]
\centering
\subfloat[Neural Network ($score \approx 0.45$).]{\includegraphics[width=3.0in]{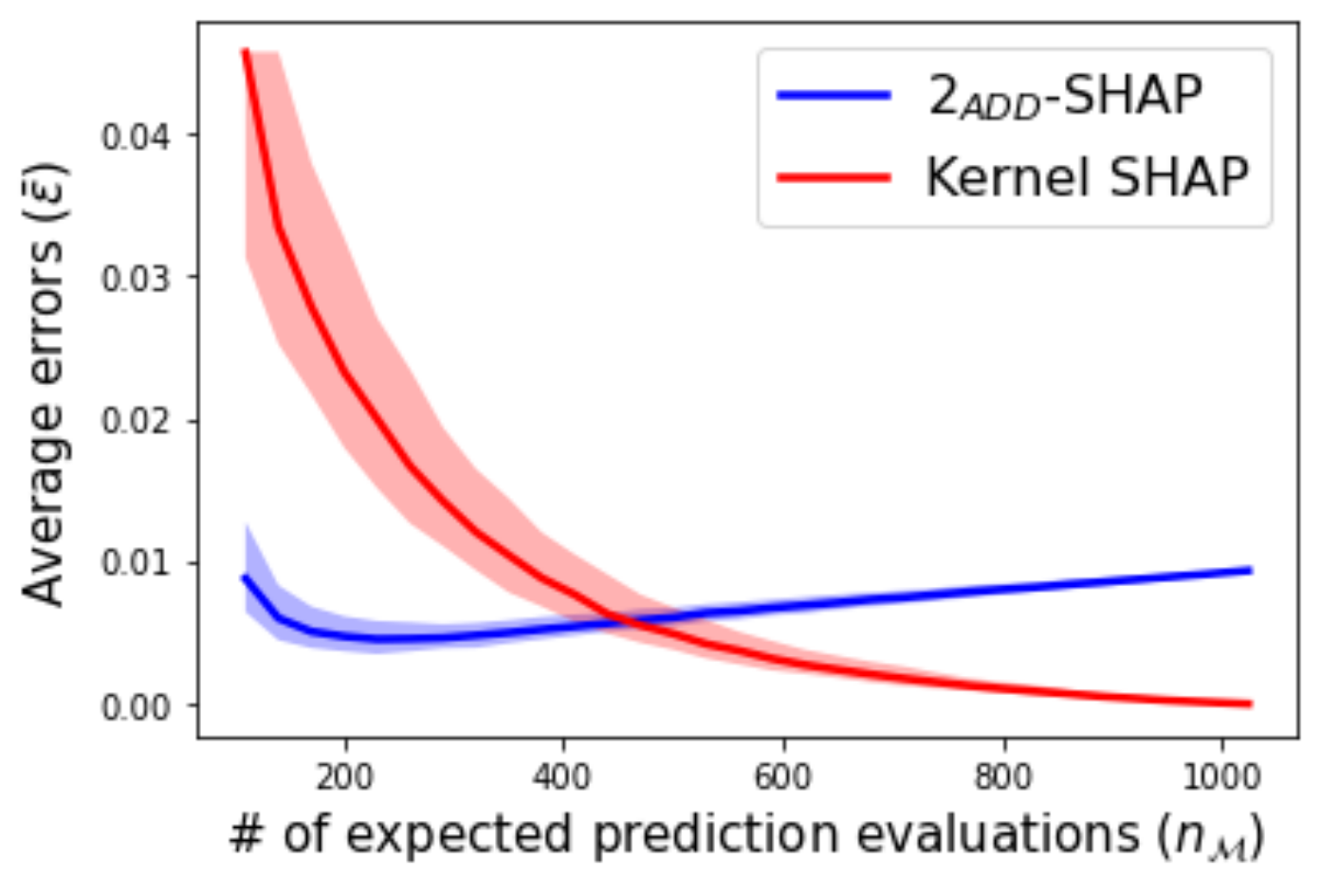}
\label{fig:convergence_diabetes_2add_mlp}}
\hfil
\subfloat[Random Forest ($score \approx 0.44$).]{\includegraphics[width=3.0in]{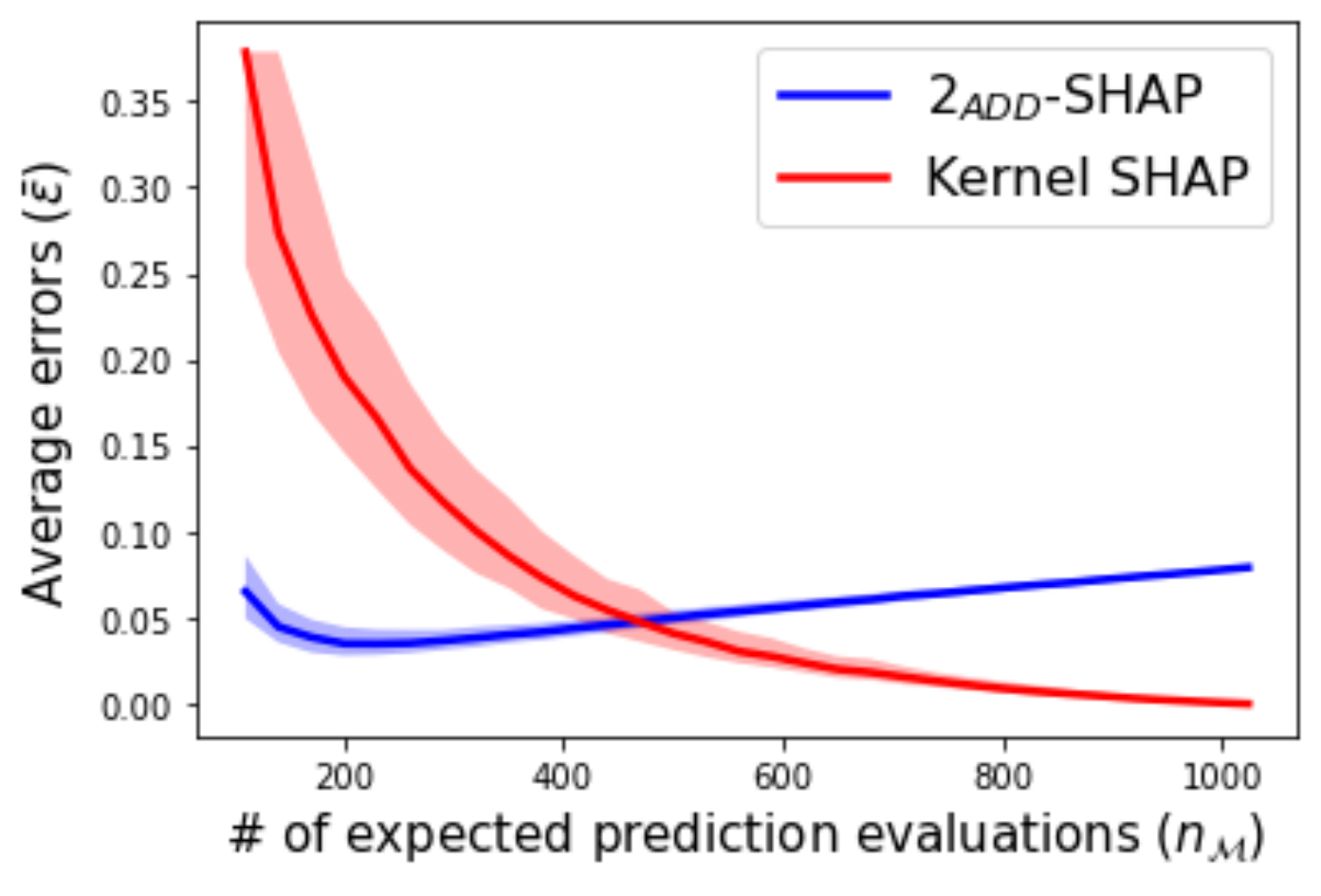}
\label{fig:convergence_diabetes_2add_rf}}
\hfill
\subfloat[Neural Network ($score \approx 0.45$).]{\includegraphics[width=3.0in]{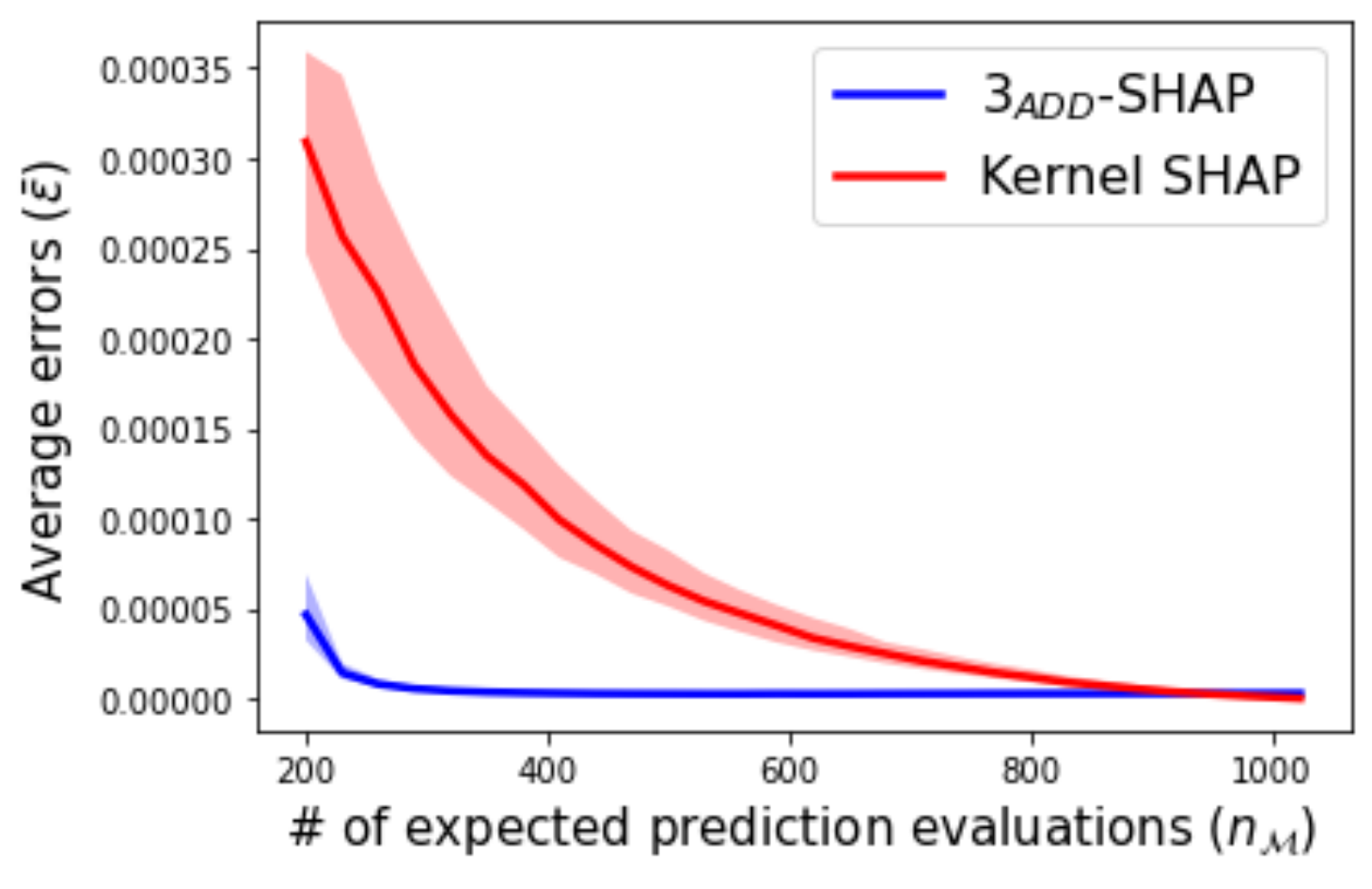}
\label{fig:convergence_diabetes_3add_mlp}}
\hfil
\subfloat[Random Forest ($score \approx 0.44$).]{\includegraphics[width=3.0in]{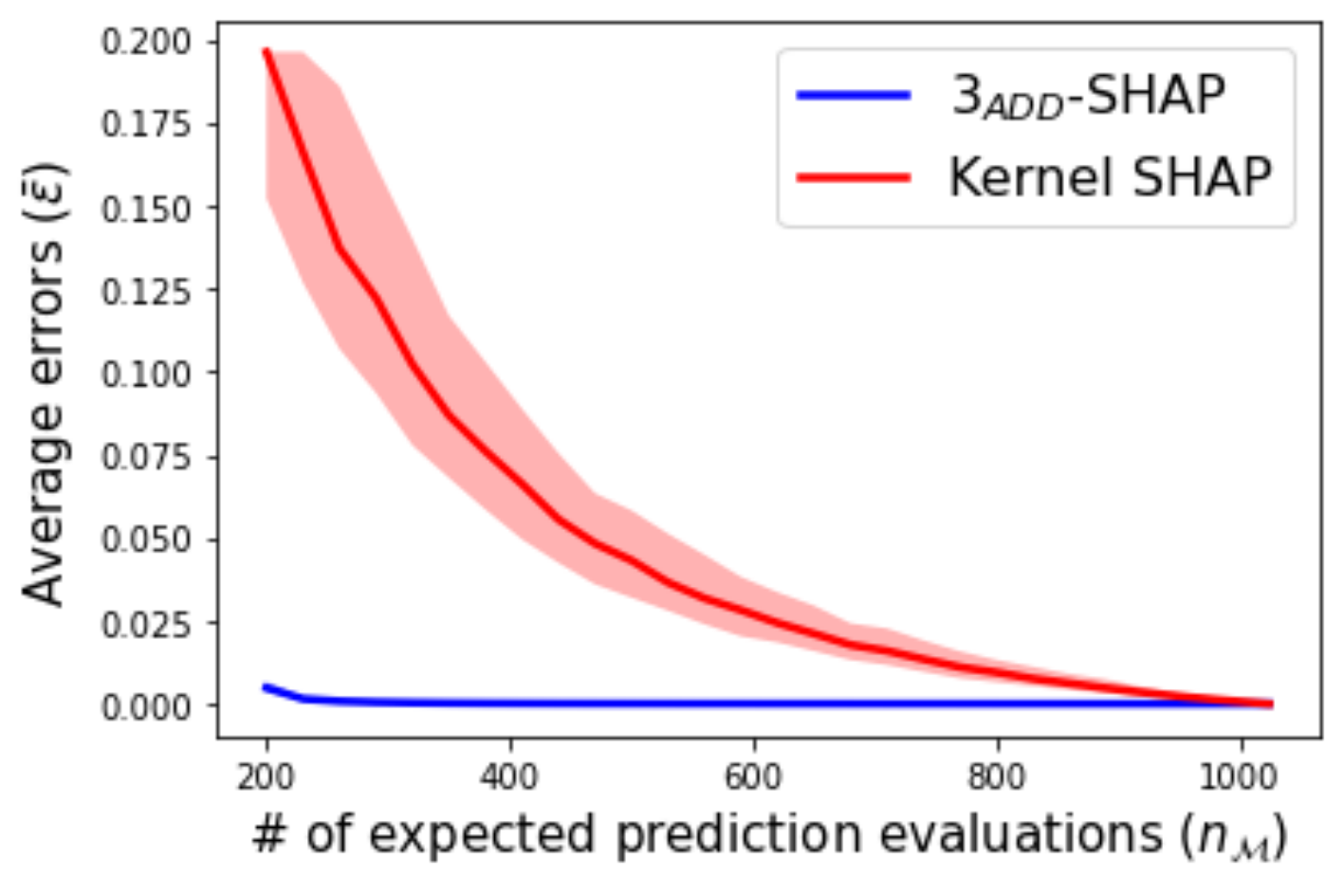}
\label{fig:convergence_diabetes_3add_rf}}
\caption{Comparison between the convergence of $2_{ADD}$-SHAP, $3_{ADD}$-SHAP and Kernel SHAP varying $n_{\mathcal{M}}$, i.e., the number of coalitions selected to calculate the expected prediction evaluations (Diabetes dataset). For both Neural Network and Random Forest, the $score$ indicates the coefficient of determination of the predicted outcomes given the test samples.}
\label{fig:convergence_diabetes}
\end{figure}

\begin{figure}[!h]
\centering
\subfloat[Neural Network ($score \approx 0.73$).]{\includegraphics[width=3.0in]{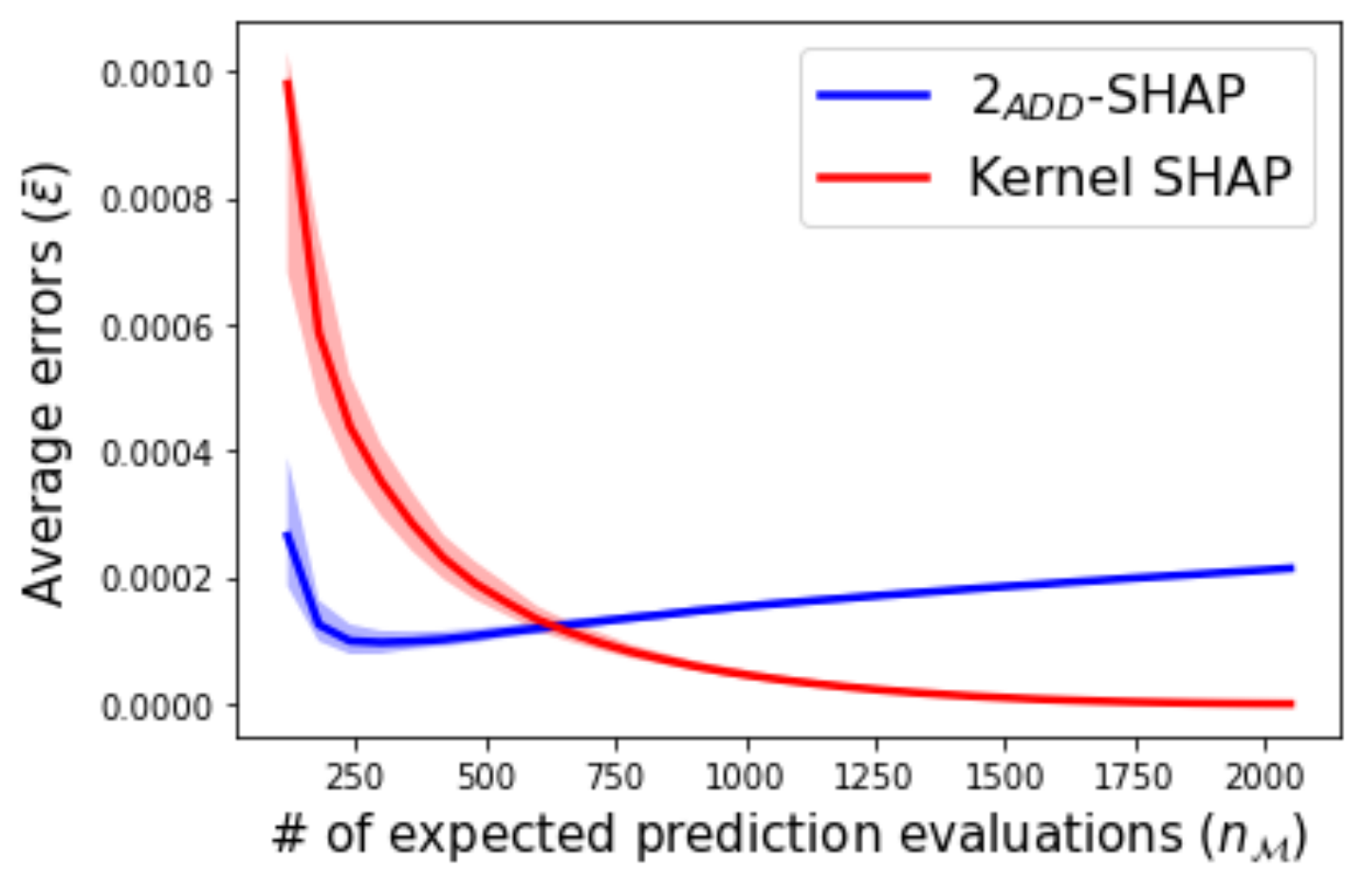}
\label{fig:convergence_wine_2add_mlp}}
\hfil
\subfloat[Random Forest ($score \approx 0.79$).]{\includegraphics[width=3.0in]{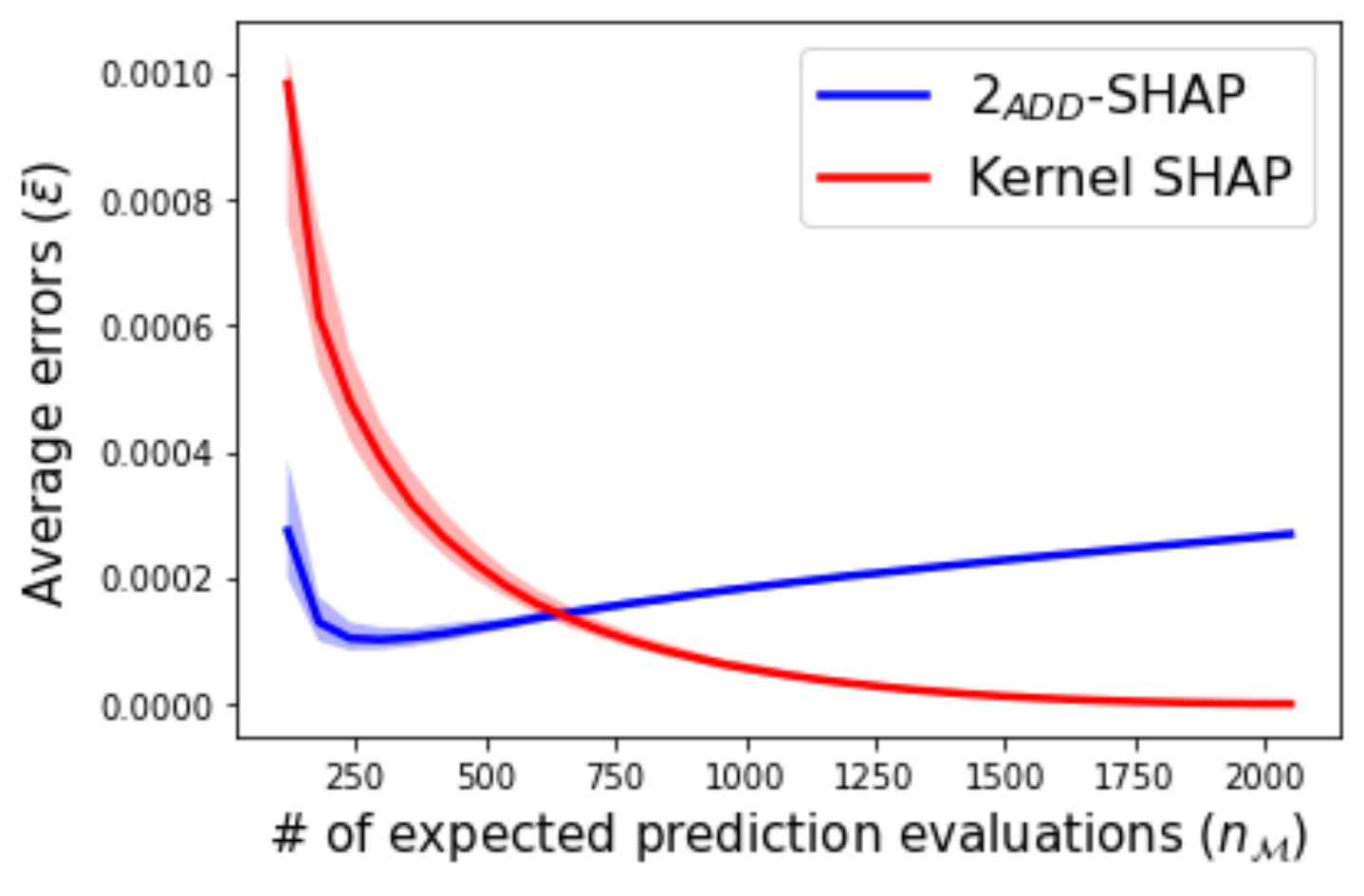}
\label{fig:convergence_wine_2add_rf}}
\hfill
\subfloat[Neural Network ($score \approx 0.73$).]{\includegraphics[width=3.0in]{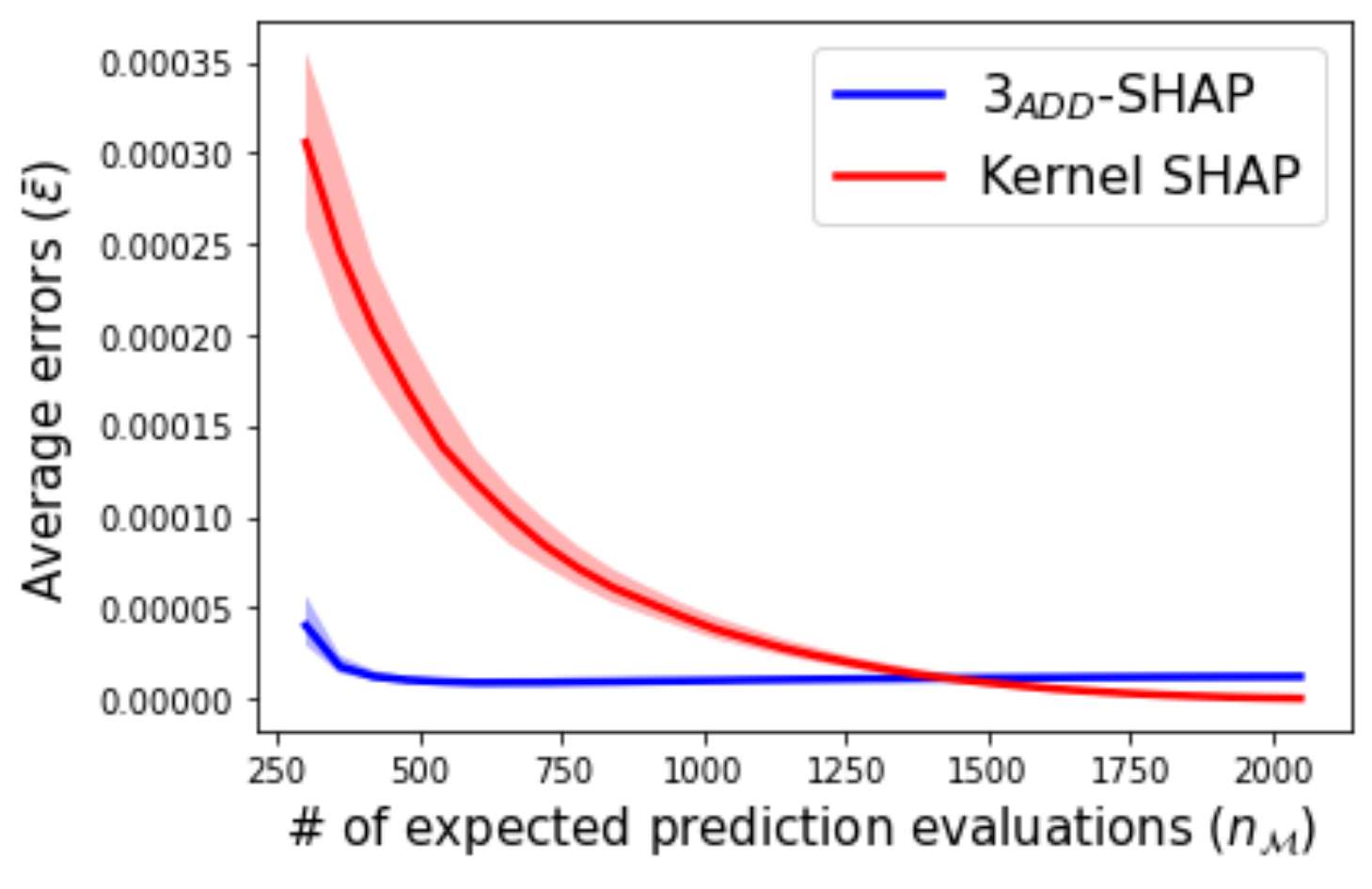}
\label{fig:convergence_wine_3add_mlp}}
\hfil
\subfloat[Random Forest ($score \approx 0.79$).]{\includegraphics[width=3.0in]{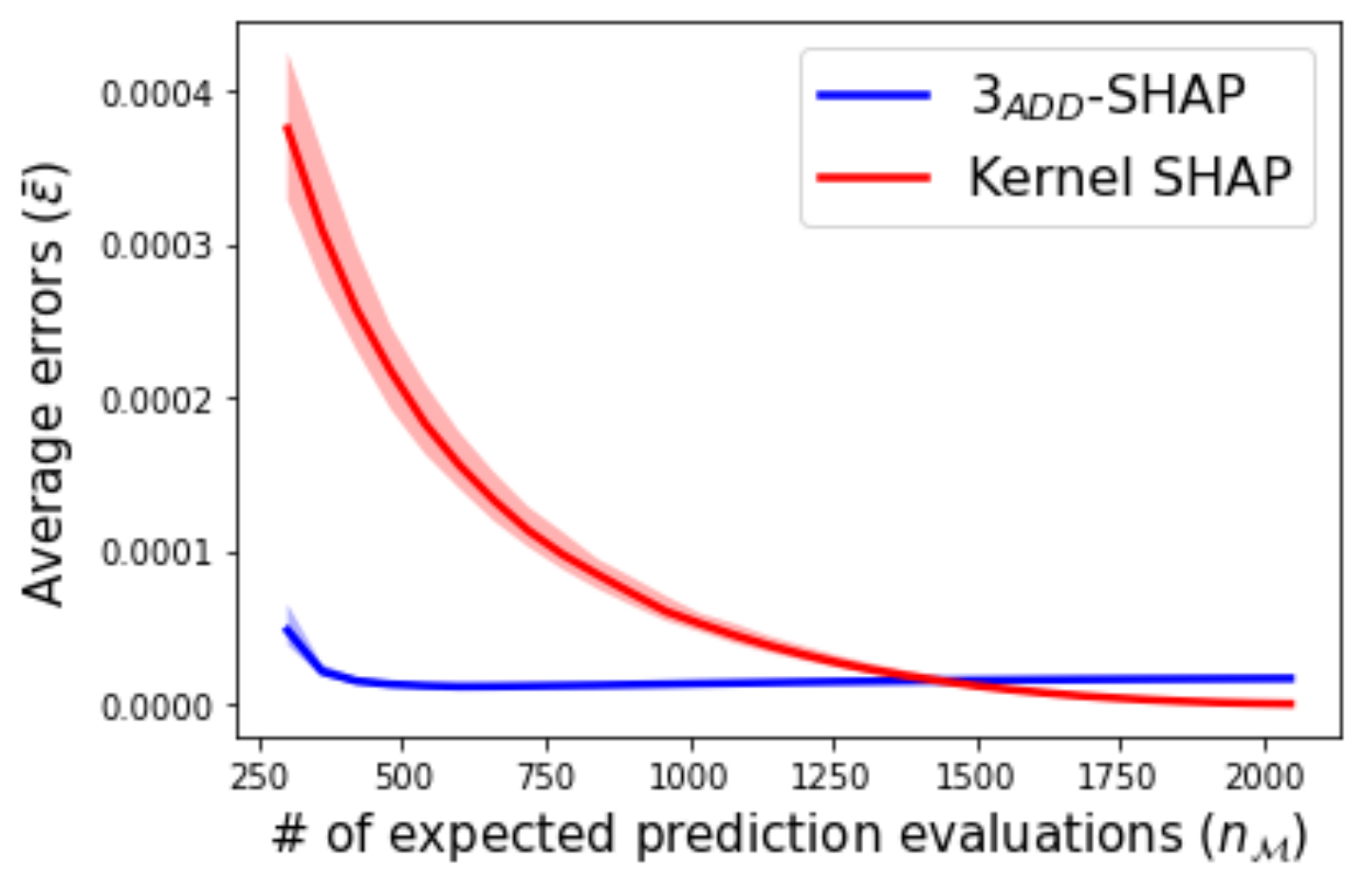}
\label{fig:convergence_wine_3add_rf}}
\hfil
\caption{Comparison between the convergence of $2_{ADD}$-SHAP, $3_{ADD}$-SHAP and Kernel SHAP varying $n_{\mathcal{M}}$, i.e., the number of coalitions selected to calculate the expected prediction evaluations (Red Wine dataset). For both Neural Network and Random Forest, the $score$ indicates the accuracy given the test samples.}
\label{fig:convergence_wine}
\end{figure}

\subsection{Experiment comparing the obtained SHAP values}
\label{subsec:exp2}

In this experiment, we compare the obtained SHAP values with the exact ones. For an instance of interest among the test data, we use the previous experiment and select the SHAP values that lead to the median error over all the simulations. For ease of visualization, we only plotted the five attributes that contribute the most (either positively or negatively) according to the exact SHAP values. As an illustrative example and without loss of generality, we selected a test sample $\mathbf{x}^*$ from the Diabetes dataset that has the attributes values described in Table~\ref{tab:diabetes_samp} (recall that this dataset is already centered with zero mean). The predicted measure of diabetes progression is equal to 84, which is less than the overall expected prediction provided by both Neural Networks and Random Forest (154.92 and 153.81, respectively). This means that the SHAP values help to explain, for the instance of interest $\mathbf{x}^*$, how each attribute value contributes to decrease the diabetes progression measure from the overall prediction until the actual 84.

\begin{table}[!h]
  \begin{center}
  \caption{Summary of the selected test sample - Diabetes dataset.}\label{tab:diabetes_samp}
  {
  \renewcommand{\arraystretch}{1.0}
	\small
  \begin{tabular}{cccccccc}
	\cline{1-2}\cline{4-5}\cline{7-8}
\textbf{Attributes} & \textbf{Values} &  & \textbf{Attributes} & \textbf{Values} &  & \textbf{Attributes} & \textbf{Values} \\
		\cline{1-2}\cline{4-5}\cline{7-8}
		age  & $0.009$ &  & blood serum 1 & $0.099$ &  & blood serum 5 & $-0.021$ \\
		\cline{1-2}\cline{4-5}\cline{7-8}
		sex  & $-0.045$ &  & blood serum 2 & $0.094$ &  & blood serum 6 & $0.007$ \\
		 \cline{1-2}\cline{4-5}\cline{7-8}
		body mass index & $-0.024$ &  & blood serum 3 & $0.071$ &  &  &  \\
		\cline{1-2}\cline{4-5}\cline{7-8}
		average blood pressure & $-0.026$ &  & blood serum 4 & $-0.002$ &  &  &  \\
		\cline{1-2}\cline{4-5}\cline{7-8}
  \end{tabular}
  }
  \end{center}
\end{table}

Figure~\ref{fig:shapley_diabetes} presents the estimated SHAP values when using $n_{\mathcal{M}}=290$, $n_{\mathcal{M}}=590$, $n_{\mathcal{M}}=890$ different coalitions of attributes to calculate the expected prediction evaluations. As a first remark, we note that the estimated SHAP values (specially the illustrated five ones) for the Neural Network (Figures~\ref{fig:shapley_diabetes_5_290_mlp},~\ref{fig:shapley_diabetes_5_590_mlp} and~\ref{fig:shapley_diabetes_5_890_mlp}) practically do not change regardless the number of predicted evaluations. All approaches led to very small errors, i.e., they could rapidly approximate the exact SHAP values associated with the Neural Networks model. For the Random Forest, we see that the contributions provided by the $3_{ADD}$-SHAP are close to the exact ones even with small number of predicted evaluations (see Figure~\ref{fig:shapley_diabetes_5_290_rf}). As one increases the number of evaluations, the Kernel SHAP converges to the exact SHAP values.

\begin{figure}[!h]
\centering
\subfloat[Neural Network, $n_{\mathcal{M}}=290$.]{\includegraphics[width=3.1in]{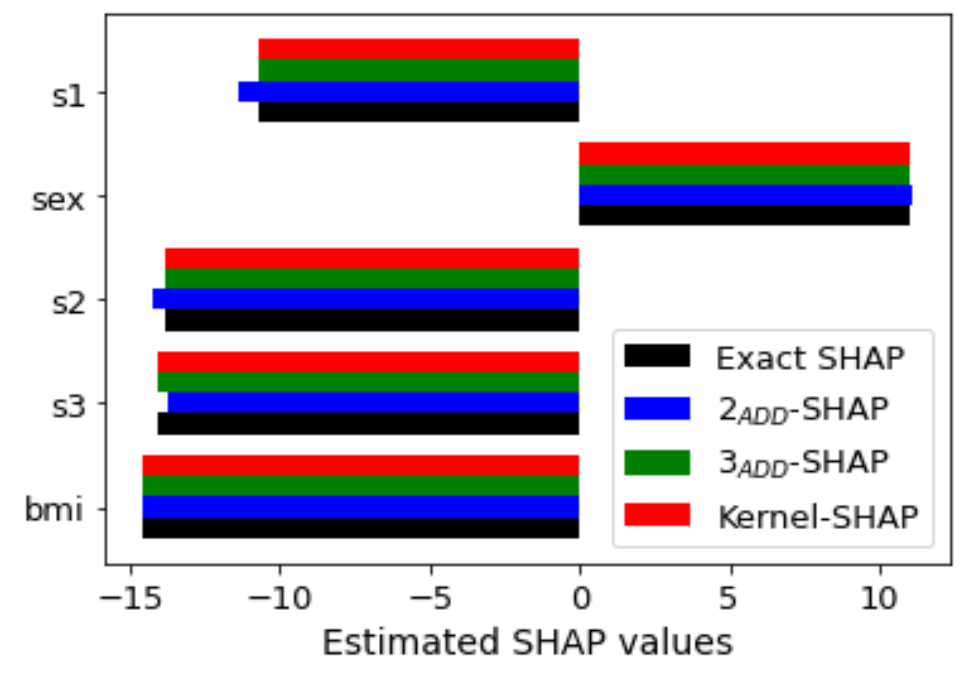}
\label{fig:shapley_diabetes_5_290_mlp}}
\hfil
\subfloat[Random Forest, $n_{\mathcal{M}}=290$.]{\includegraphics[width=3.1in]{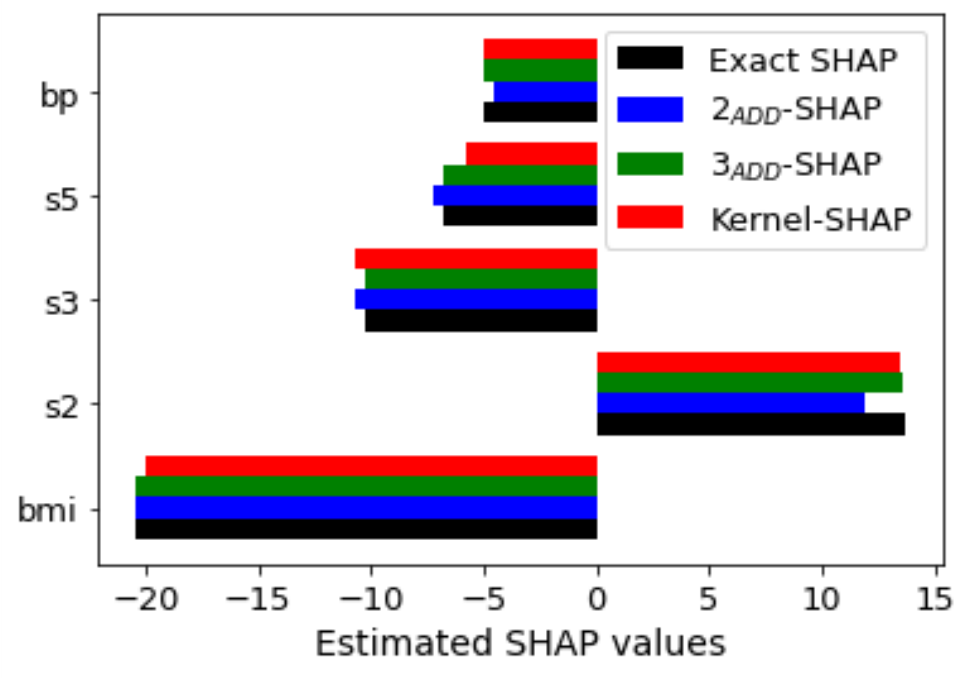}
\label{fig:shapley_diabetes_5_290_rf}}
\hfill
\subfloat[Neural Network, $n_{\mathcal{M}}=590$.]{\includegraphics[width=3.1in]{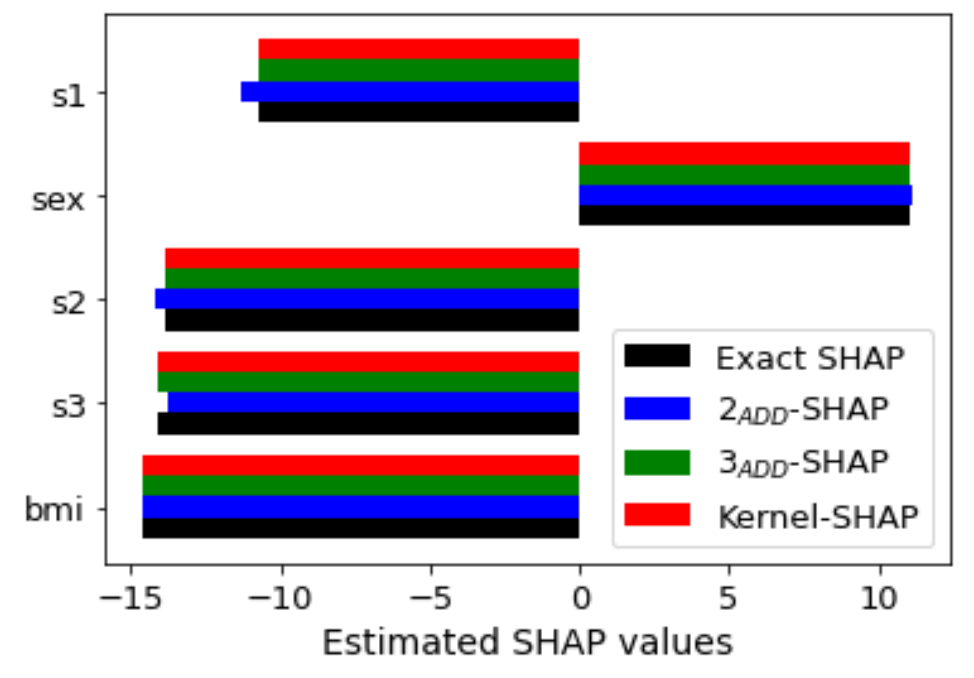}
\label{fig:shapley_diabetes_5_590_mlp}}
\hfil
\subfloat[Random Forest, $n_{\mathcal{M}}=590$.]{\includegraphics[width=3.1in]{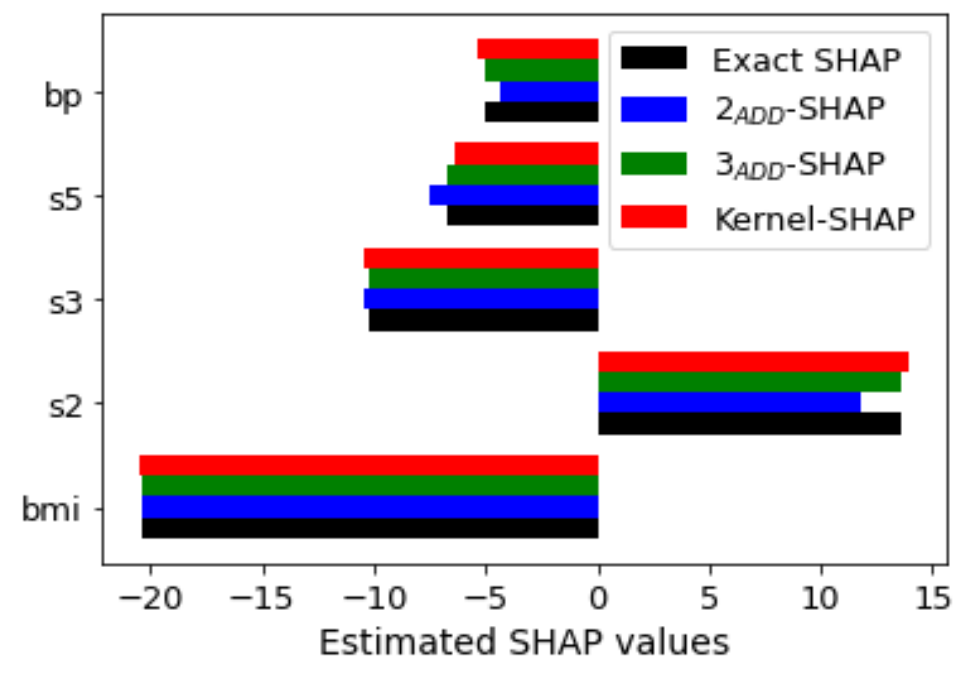}
\label{fig:shapley_diabetes_5_590_rf}}
\hfil
\subfloat[Neural Network, $n_{\mathcal{M}}=890$.]{\includegraphics[width=3.1in]{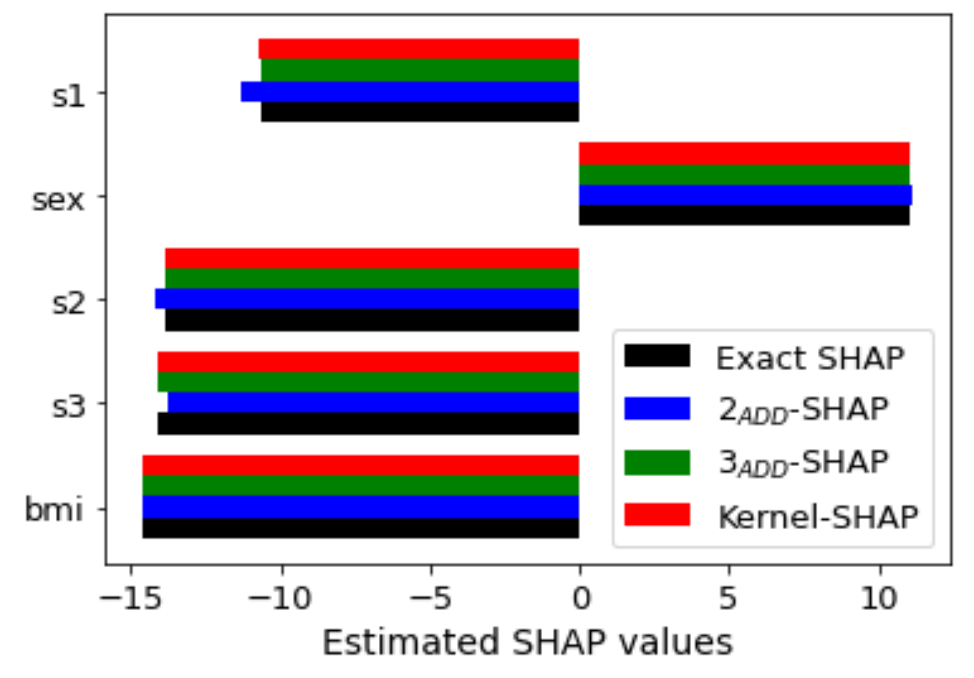}
\label{fig:shapley_diabetes_5_890_mlp}}
\hfil
\subfloat[Random Forest, $n_{\mathcal{M}}=890$.]{\includegraphics[width=3.1in]{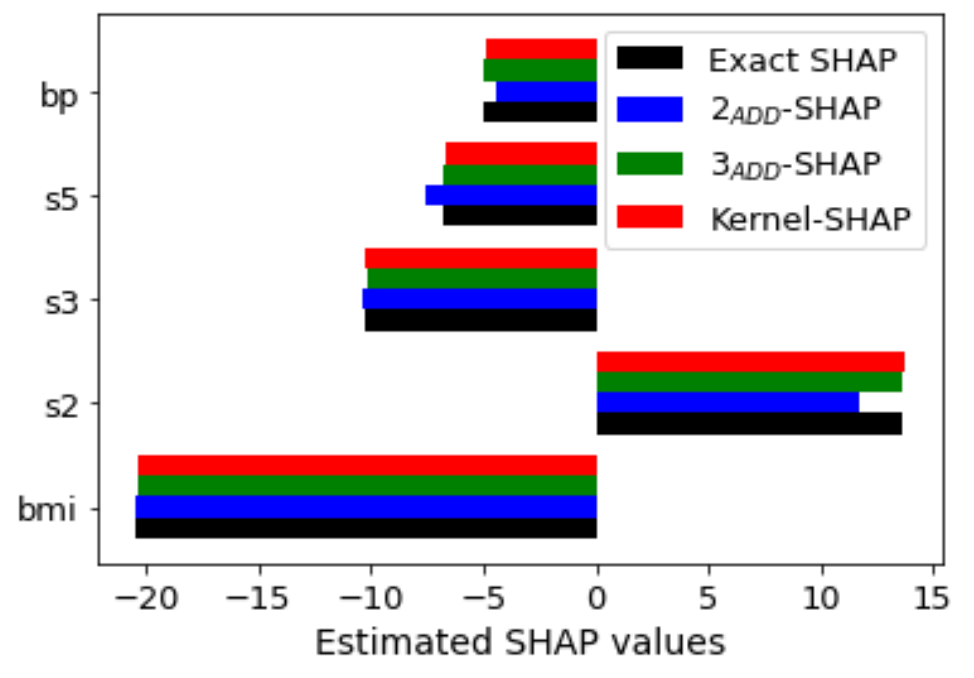}
\label{fig:shapley_diabetes_5_890_rf}}
\hfil
\caption{Comparison between the estimated SHAP values provided by the $2_{ADD}$-SHAP, $3_{ADD}$-SHAP and Kernel SHAP for different machine learning models and varying the number of coalitions used to calculate the expected prediction evaluations (Diabetes dataset).}
\label{fig:shapley_diabetes}
\end{figure}

Regarding the Red Wine dataset, we selected as an illustrative example a test sample classified as a good wine. The attributes values described in Table~\ref{tab:wine_samp}. The overall expected probability prediction for class 1 (good wine) for both Neural Networks and Random Forest is approximately 0.53. In this case, the SHAP values indicates the contributions of attributes that increase the probability of being classified as a good wine from the overall expected probability until the actual classification (class value equals to 1).

\begin{table}[!h]
  \begin{center}
  \caption{Summary of the selected test sample - Red Wine dataset.}\label{tab:wine_samp}
  {
  \renewcommand{\arraystretch}{1.0}
	\small
  \begin{tabular}{cccccccc}
	\cline{1-2}\cline{4-5}\cline{7-8}
\textbf{Attributes} & \textbf{Values} &  & \textbf{Attributes} & \textbf{Values} &  & \textbf{Attributes} & \textbf{Values} \\
		\cline{1-2}\cline{4-5}\cline{7-8}
		fixed acidity  & $9.4$ &  & chlorides & $0.08$ &  & pH & $3.15$ \\
		\cline{1-2}\cline{4-5}\cline{7-8}
		volatile acidity  & $0.3$ &  & free sulfur dioxide & $6$ &  & sulphates & $0.92$ \\
		 \cline{1-2}\cline{4-5}\cline{7-8}
		citric acid & $0.56$ &  & total sulfur dioxide & $17$ &  & alcohol & $11.7$ \\
		\cline{1-2}\cline{4-5}\cline{7-8}
		residual sugar & $2.8$ &  & density & $0.9964$ &  &  &  \\
		\cline{1-2}\cline{4-5}\cline{7-8}
  \end{tabular}
  }
  \end{center}
\end{table}

Figure~\ref{fig:shapley_wine} presents the estimated SHAP values when using $n_{\mathcal{M}}=420$, $n_{\mathcal{M}}=1020$ and $n_{\mathcal{M}}=1800$ coalitions of attributes. As in the previous dataset, we can see that, even for a reduced number of samples, the $3_{ADD}$-SHAP converges faster to the exact SHAP values. When the number of expected prediction evaluations increases, the Kernel SHAP converges to the exact SHAP values while the $2_{ADD}$-SHAP slightly diverges.

\begin{figure}[!h]
\centering
\subfloat[Neural Networks, $n_{\mathcal{M}}=420$.]{\includegraphics[width=3.1in]{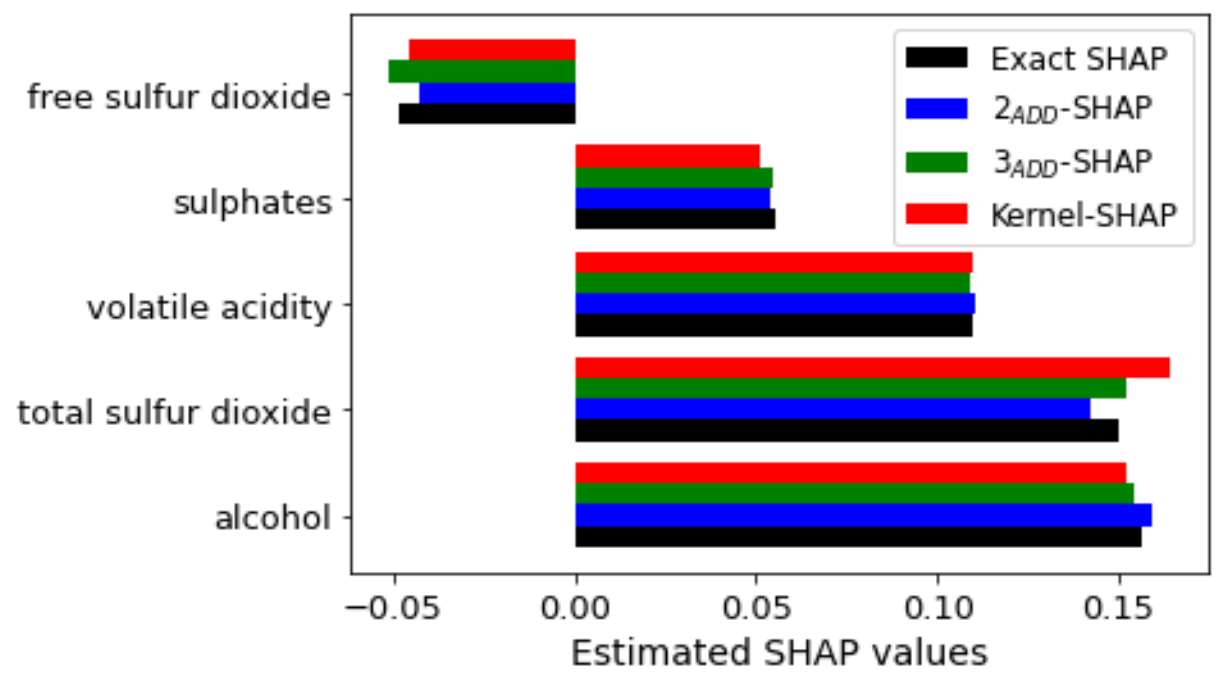}
\label{fig:shapley_wine_18_420_mlp}}
\hfil
\subfloat[Random Forest, $n_{\mathcal{M}}=420$.]{\includegraphics[width=3.1in]{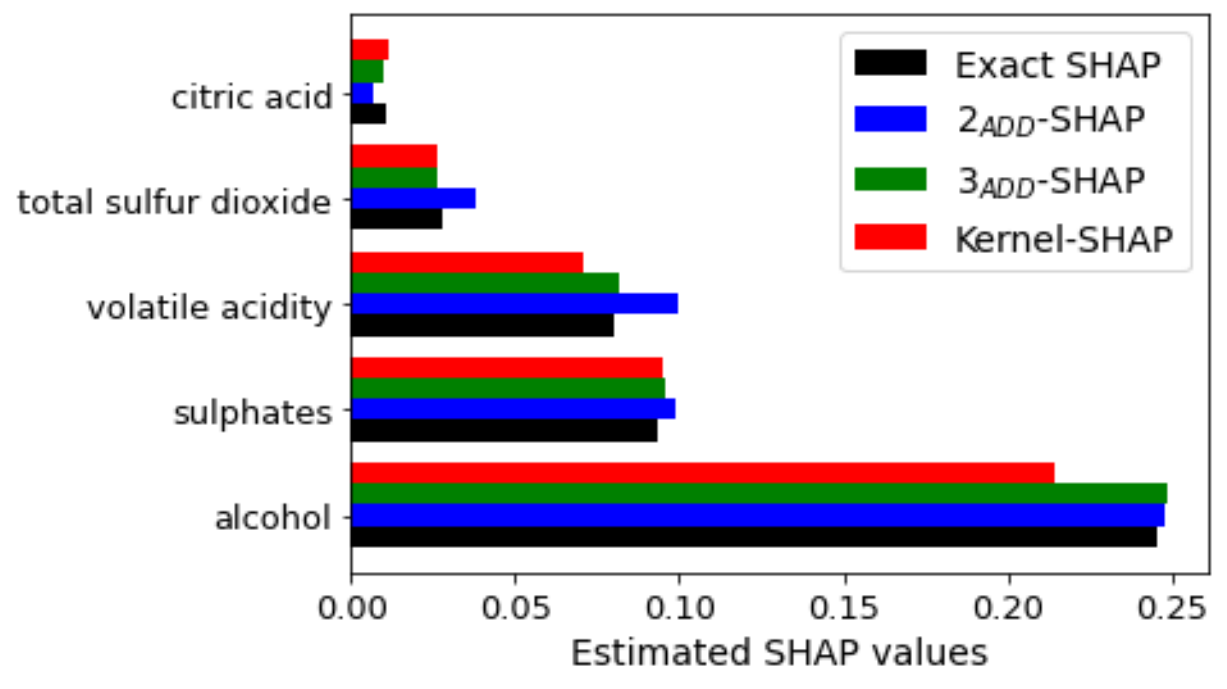}
\label{fig:shapley_wine_18_420_rf}}
\hfill
\subfloat[Neural Networks, $n_{\mathcal{M}}=1020$.]{\includegraphics[width=3.1in]{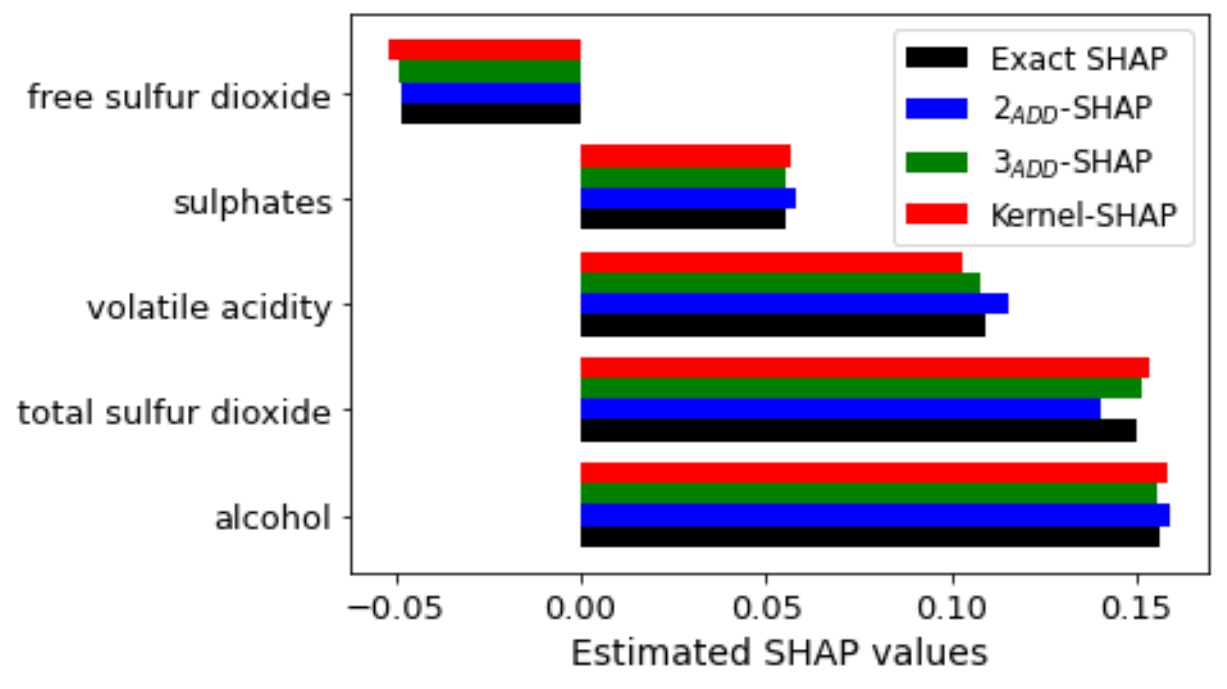}
\label{fig:shapley_wine_18_1020_mlp}}
\hfil
\subfloat[Random Forests, $n_{\mathcal{M}}=1020$.]{\includegraphics[width=3.1in]{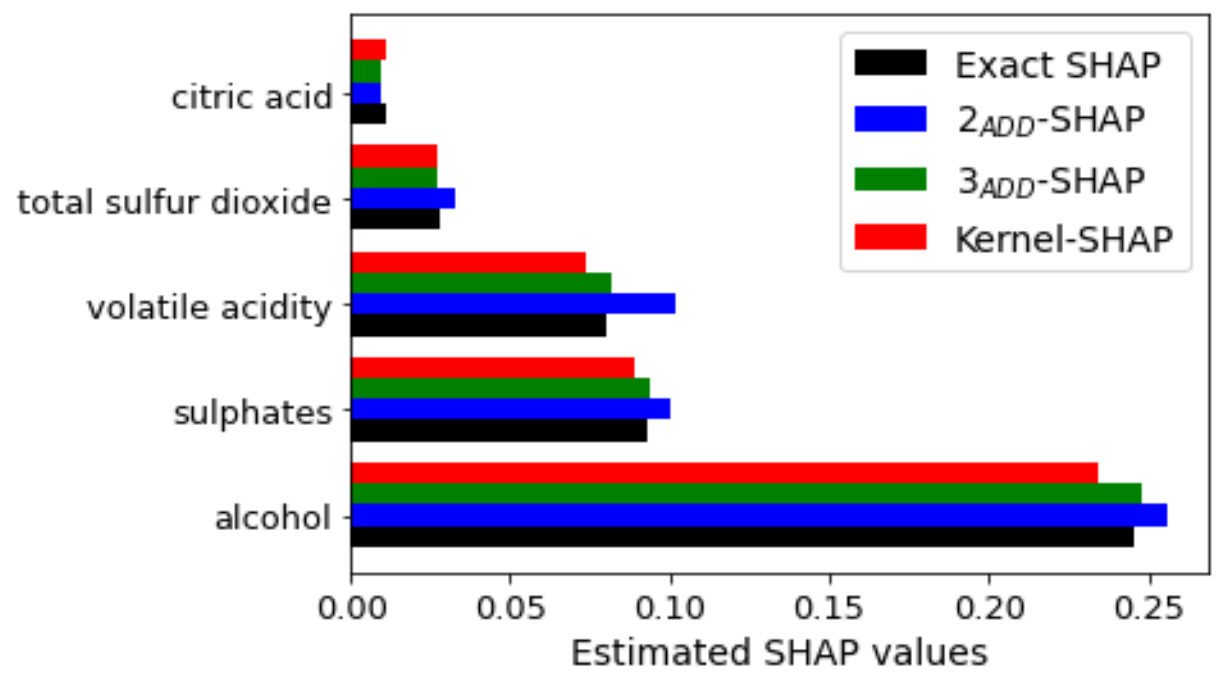}
\label{fig:shapley_wine_18_1020_rf}}
\hfil
\subfloat[Neural Networks, $n_{\mathcal{M}}=1800$.]{\includegraphics[width=3.1in]{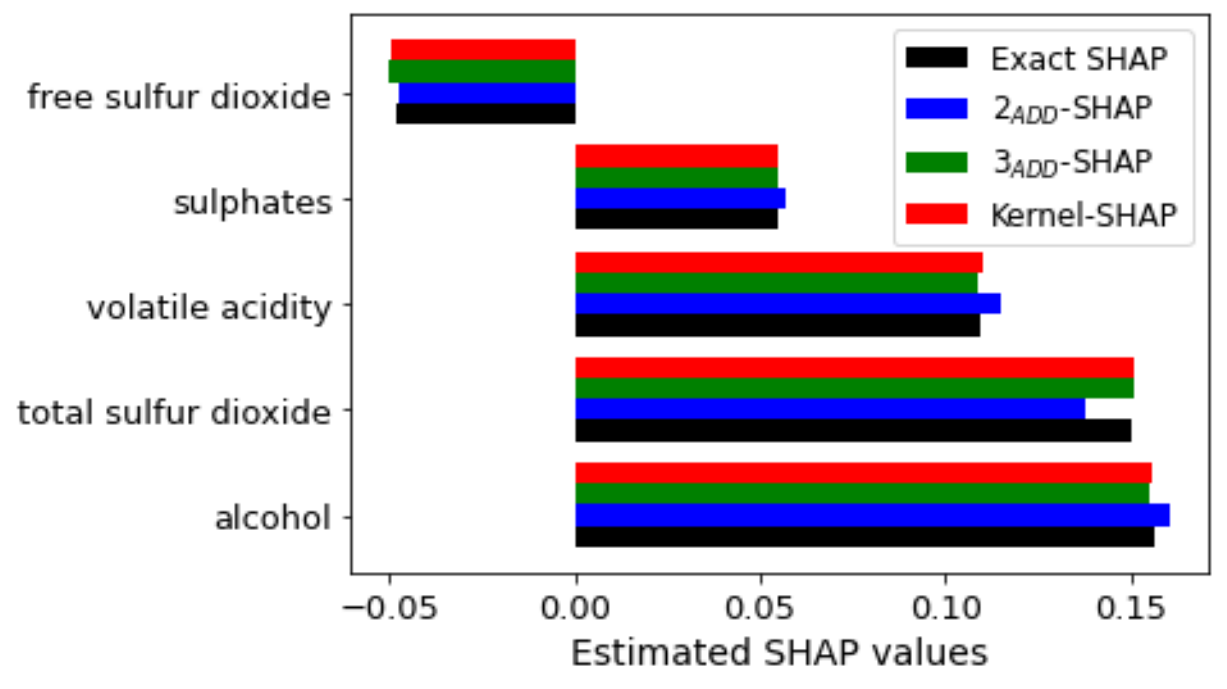}
\label{fig:shapley_wine_18_1800_mlp}}
\hfil
\subfloat[Random Forest, $n_{\mathcal{M}}=1800$.]{\includegraphics[width=3.1in]{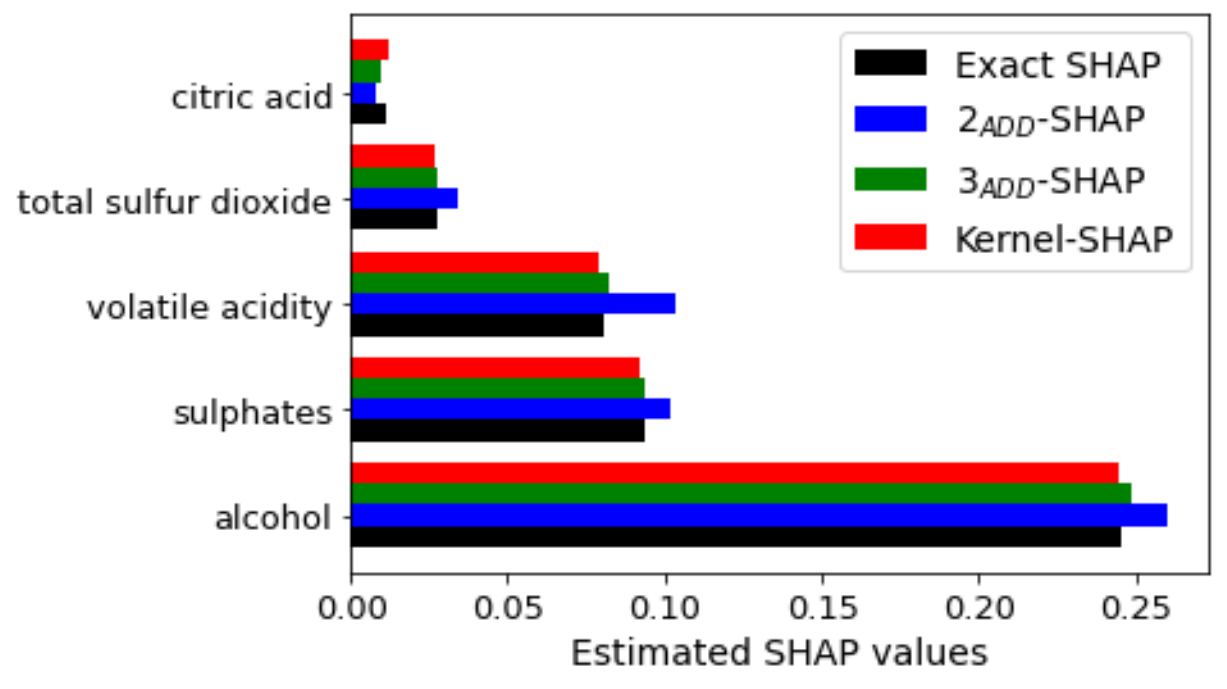}
\label{fig:shapley_wine_18_1800_rf}}
\hfil
\caption{Comparison between the estimated SHAP values provided by the $2_{ADD}$-SHAP, $3_{ADD}$-SHAP and Kernel SHAP for different machine learning models and varying the number of coalitions used to calculate the expected prediction evaluations (Red Wine dataset).}
\label{fig:shapley_wine}
\end{figure}

\subsection{Illustrative example and results visualization}
\label{subsec:exp3}


The purpose of this last experiment is to apply our proposal to visualize the attributes contribution towards the actual predicted outcome. We use as an illustrative example the Red Wine dataset and applied the $3_{ADD}$-SHAP. We also consider the test sample used in the previous experiment, which is classified as a good wine. Based on 1500 predicted evaluations and using the Random Forest, the contributions of attributes are presented in Figure~\ref{fig:example_wine_shapley}. Note that there are three attributes that contribute the most into the predicted outcome: alcohol, sulphates and volatile acidity. They are all positively contributing to predict the sample as a good wine.

\begin{figure*}[h!t]
\begin{centering}
\includegraphics[width=0.60\textwidth]{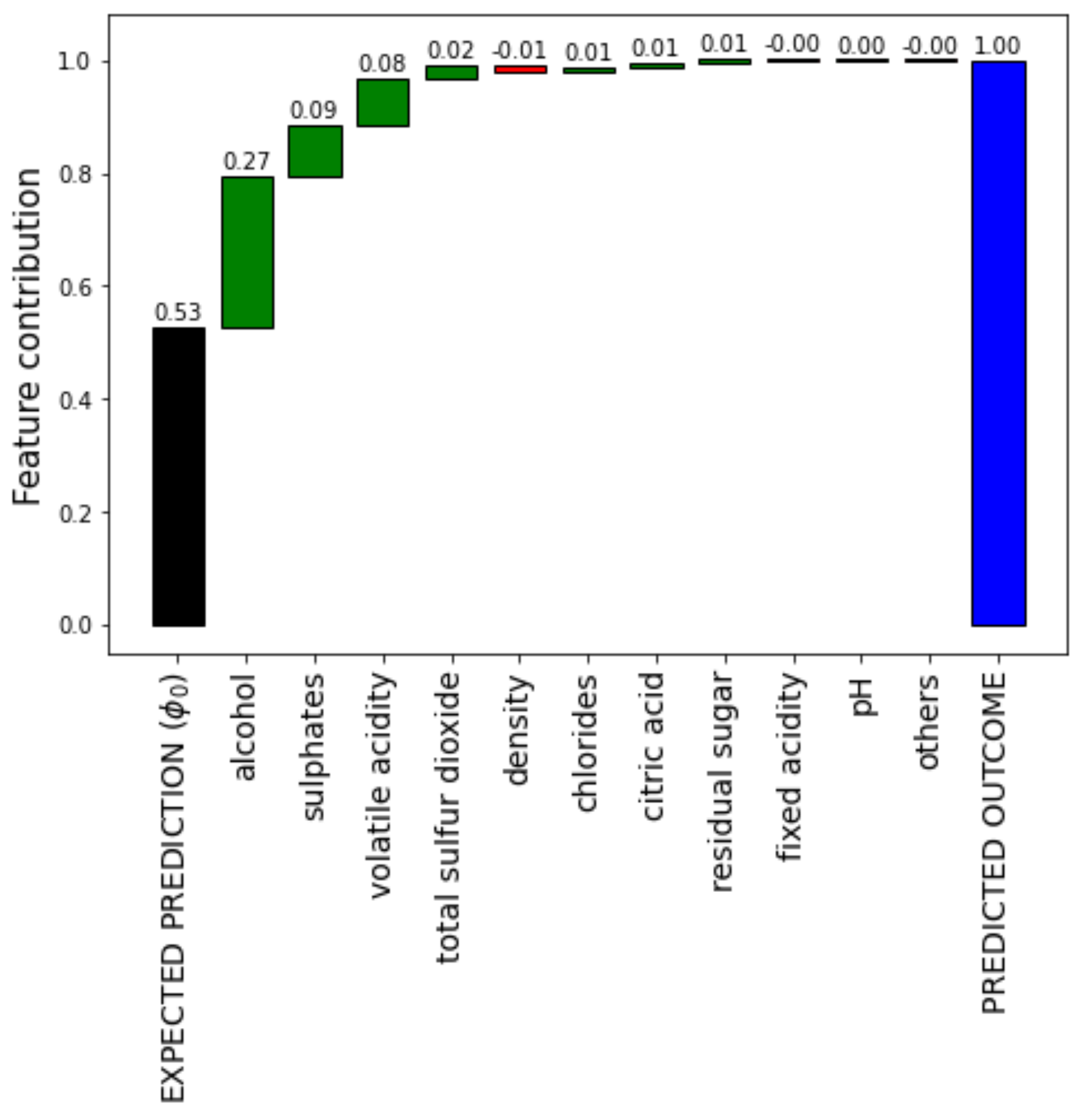} 
\par\end{centering}
\centering{}\caption{Attributes contribution towards the predicted outcome - $3_{ADD}$-SHAP and Red Wine dataset.
\label{fig:example_wine_shapley}}
\end{figure*}

Recall that, more than the contribution of features, our proposal automatically provides the interaction degree between them. We highlight that these interaction effects do not come up with the original Kernel SHAP formulation. Indeed, further adaptations must be made in Kernel SHAP in order to retrieve the interaction effects~\citep{Lundberg2020}. Figure~\ref{fig:example_wine_interaction} shows the interaction degree between attributes for the considered test sample. It indicates that, although volatile acidity, sulphates and alcohol (attributes 1, 9 and 10, respectively) contributes the most to the predicted outcome, there are negative interactions between alcohol and both volatile acidity and sulphates. This suggests that there are some redundancies between alcohol and the other two attributes when predicting the sample as a good wine.

\begin{figure*}[h!t]
\begin{centering}
\includegraphics[width=0.50\textwidth]{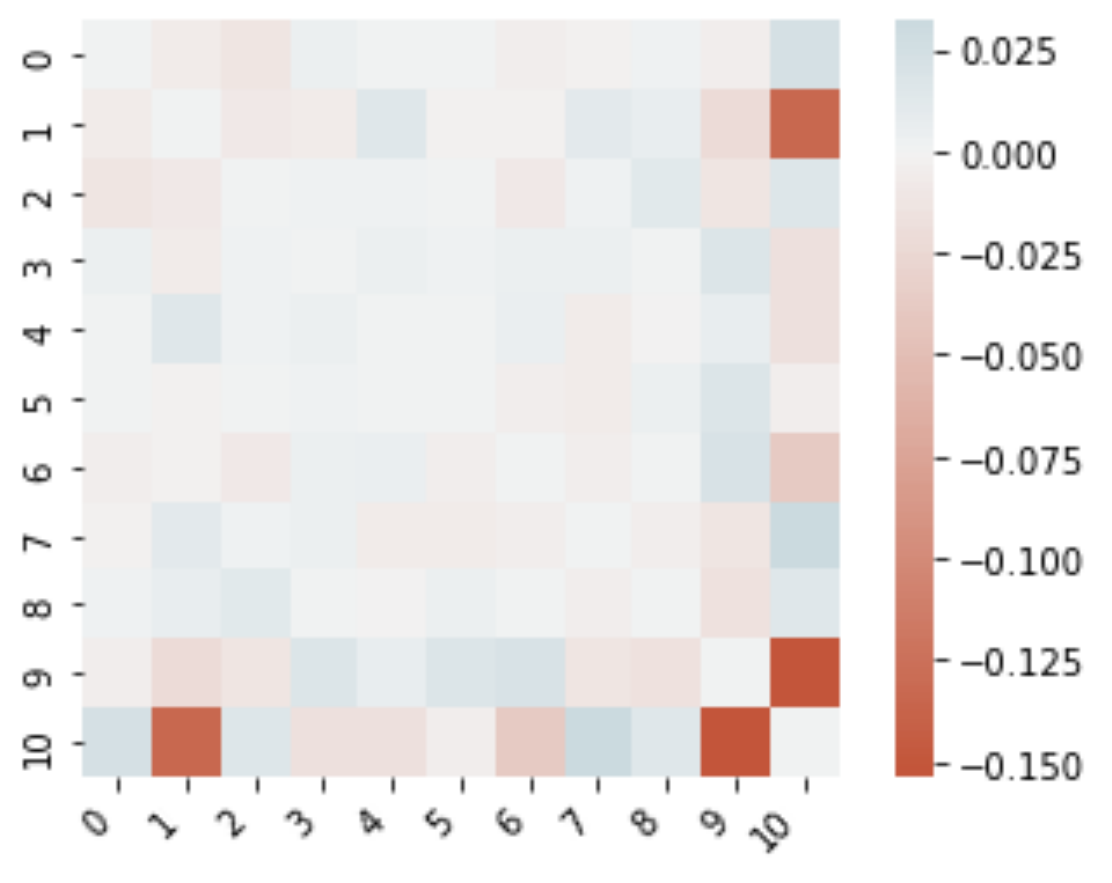} 
\par\end{centering}
\centering{}\caption{Interaction degree between attributes - $3_{ADD}$-SHAP and Red Wine dataset.
\label{fig:example_wine_interaction}}
\end{figure*}

\section{Conclusions and future perspectives}
\label{concl}

Interpretability in machine learning has become as important as accuracy in real problems. For instance, even if there is a correct classification (e.g., a denied credit), the explanation about how this result was achieved is required to ensure the model trustfulness. A very famous model-agnostic algorithm for machine learning interpretability is the SHAP method. Based on the Shapley values, the SHAP method indicates the contribution of each attribute in the predicted outcome. For this purpose, we look at the machine learning task as a cooperative game theory problem and calculate the marginal contribution of each attribute by taking the predicted outcomes of all possible coalitions of attributes. A point of attention in this calculation is that, as the number of predicted outcomes evaluations exponentially increases with the number of attributes, one may not be able to obtain the exact SHAP values.

In order to reduce the computational effort of SHAP method, the Kernel SHAP emerges as a clever strategy to approximate the SHAP values. However, its formulation is not easy to follow and any further considerations about the modeled game is assumed when approximating the SHAP values. In this paper, we first proposed a straightforward Choquet integral-based formulation for local interpretability. As the parameters used in the Choquet integral are directly associated with the Shapley values, our formulation also leads to the SHAP values. Therefore, we can also exploit the benefits of the SHAP values when interpreting local predictions. Moreover, our formulation also provides the interaction effects between attributes without further adaptations in the algorithm. Therefore, we can interpret the marginal contribution of each attribute towards the predicted outcome and how they interact between them.

As a second contribution, we exploit the concept of $k$-additive games. The use of $k$-additive models has revealed to be useful in multicriteria decision making problems in order to reduce the number of parameters in capacity-based aggregation functions (such as the Choquet integral) while keeping a good level of flexibility in data modeling. Therefore, as attested in the numerical experiments, when adopting $k$-additive games (specially the $3$-additive, which leads to the proposed $3_{ADD}$-SHAP), we could approximate the SHAP values using less predicted outcomes evaluations in comparison with the Kernel SHAP. As one reduced the number of parameters in the Choquet integral formulation, one avoided over-parametrization in scenarios with a low number of predicted outcomes evaluations. On the other hand, as we restricted the modeling data domain, in the scenario with all evaluations the proposed $k_{ADD}$-SHAP may slightly diverge from the exact SHAP values. However, as could be seen in the experiments, this difference is very low (mainly for the $3_{ADD}$-SHAP) and it does not affect the interpretability.

Future works include to extend the proposed approach when assuming that the attributes are dependent. In such a scenario, the formulation could be adjusted in order to better approximate the Shapley values~\citep{Aas2021}. Another perspective consists in evaluating the use of other game-based aggregation functions to deal with local interpretability. However, as some of them do not ensure the efficiency property, one must be careful in how one can apply them in the context of machine learning in a way that the feature attribution makes sense for local or global interpretability.

\section*{Acknowledgments}
Work supported by S\~{a}o Paulo Research Foundation (FAPESP) under the grants \#2020/09838-0 (BI0S - Brazilian Institute of Data Science), \#2020/10572-5 and \#2021/11086-0.

\section*{Appendix A}

We here describe the desired properties satisfied by SHAP values, which are derived from the Shapley values properties~\citep{Shapley1953,Young1985}. Recall that $f(\mathbf{x})$ is the predicted outcome of a trained model $f(\cdot)$, $\mathbf{x}$ is the instance to be explained and $\mathbf{z}'$ is a binary vector. The proofs are provided in the original SHAP paper~\citep{Lundberg2017}.

\begin{Properties}
  \item \textbf{Local accuracy (or efficiency)} \\
  \begin{equation}
  f(\mathbf{x}) = \phi_0 + \sum_{j=1}^m \phi_j(f,\mathbf{x})
  \end{equation}
  The local accuracy property states that the predicted outcome $f(\mathbf{x})$ can be decomposed by the sum of the SHAP values and the overall expected prediction $\phi_0$.
  
  \item \textbf{Missingness} \\
  If, for all subset of attributes represented by the coalition $\mathbf{z}'$,
  \begin{equation}
  f\left( h_{\mathbf{x}}(\mathbf{z}')\right) = f\left( h_{\mathbf{x}}(\mathbf{z}'\backslash j)\right),
  \end{equation}
  then $\phi_j(f,\mathbf{x}) = 0$. This property states that, if adding attribute $j$ into the coalition the expected prediction remains the same, the marginal contribution of such an attribute is null.
  
  \item \textbf{Consistency (or monotonicity)} \\
  For any two models $f(\cdot)$ and $f'(\cdot)$, if
  \begin{equation}
  f'\left( h_{\mathbf{x}}(\mathbf{z}')\right) - f'\left( h_{\mathbf{x}}(\mathbf{z}' \backslash j)\right) \geq f\left( h_{\mathbf{x}}(\mathbf{z}')\right) - f\left( h_{\mathbf{x}}(\mathbf{z}'\backslash j)\right)
  \end{equation}
  for any binary vector $\mathbf{z}' \in \left\{0,1\right\}^m$, then $\phi_j(f',\mathbf{x}) \geq \phi_j(f,\mathbf{x})$. The consistency property states that, if one changes the trained model and the contribution of an attribute $j$ increases or stays the same regardless of the other inputs, the marginal contribution of such an attribute should not decrease.
\end{Properties}

\bibliographystyle{model5-names}
\biboptions{authoryear}
\bibliography{_references_choquet}

\begin{thebibliography}{39}
\expandafter\ifx\csname natexlab\endcsname\relax\def\natexlab#1{#1}\fi
\providecommand{\url}[1]{\texttt{#1}}
\providecommand{\href}[2]{#2}
\providecommand{\path}[1]{#1}
\providecommand{\DOIprefix}{doi:}
\providecommand{\ArXivprefix}{arXiv:}
\providecommand{\URLprefix}{URL: }
\providecommand{\Pubmedprefix}{pmid:}
\providecommand{\doi}[1]{\href{http://dx.doi.org/#1}{\path{#1}}}
\providecommand{\Pubmed}[1]{\href{pmid:#1}{\path{#1}}}
\providecommand{\bibinfo}[2]{#2}
\ifx\xfnm\relax \def\xfnm[#1]{\unskip,\space#1}\fi
\bibitem[{Aas et~al.(2021)Aas, Jullum \& Løland}]{Aas2021}
\bibinfo{author}{Aas, K.}, \bibinfo{author}{Jullum, M.}, \&
  \bibinfo{author}{Løland, A.} (\bibinfo{year}{2021}).
\newblock \bibinfo{title}{Explaining individual predictions when features are
  dependent: More accurate approximations to {S}hapley values}.
\newblock {\it \bibinfo{journal}{Artificial Intelligence}\/},  {\it
  \bibinfo{volume}{298}\/}, \bibinfo{pages}{103502}.
\bibitem[{Ahsan \& Siddique(2022)}]{Ahsan2022}
\bibinfo{author}{Ahsan, M.~M.}, \& \bibinfo{author}{Siddique, Z.}
  (\bibinfo{year}{2022}).
\newblock \bibinfo{title}{Machine learning-based heart disease diagnosis: A
  systematic literature review}.
\newblock {\it \bibinfo{journal}{Artificial Intelligence in Medicine}\/},  {\it
  \bibinfo{volume}{128}\/}, \bibinfo{pages}{102289}.
\bibitem[{Ben-Israel et~al.(2020)Ben-Israel, Jacobs, Casha, Lang, Ryu,
  de~Lotbiniere-Bassett \& Cadotte}]{BenIsrael2020}
\bibinfo{author}{Ben-Israel, D.}, \bibinfo{author}{Jacobs, W.~B.},
  \bibinfo{author}{Casha, S.}, \bibinfo{author}{Lang, S.},
  \bibinfo{author}{Ryu, W. H.~A.}, \bibinfo{author}{de~Lotbiniere-Bassett, M.},
  \& \bibinfo{author}{Cadotte, D.~W.} (\bibinfo{year}{2020}).
\newblock \bibinfo{title}{The impact of machine learning on patient care: A
  systematic review}.
\newblock {\it \bibinfo{journal}{Artificial Intelligence in Medicine}\/},  {\it
  \bibinfo{volume}{103}\/}, \bibinfo{pages}{101785}.
\bibitem[{Bentéjac et~al.(2021)Bentéjac, Csörgő \&
  Martínez-Muñoz}]{Bentejac2021}
\bibinfo{author}{Bentéjac, C.}, \bibinfo{author}{Csörgő, A.}, \&
  \bibinfo{author}{Martínez-Muñoz, G.} (\bibinfo{year}{2021}).
\newblock \bibinfo{title}{A comparative analysis of gradient boosting
  algorithms}.
\newblock {\it \bibinfo{journal}{Artificial Intelligence Review}\/},  {\it
  \bibinfo{volume}{54}\/}, \bibinfo{pages}{1937--1967}.
\bibitem[{Biau \& Scornet(2016)}]{Biau2016}
\bibinfo{author}{Biau, G.}, \& \bibinfo{author}{Scornet, E.}
  (\bibinfo{year}{2016}).
\newblock \bibinfo{title}{A random forest guided tour}.
\newblock {\it \bibinfo{journal}{Test}\/},  {\it \bibinfo{volume}{25}\/},
  \bibinfo{pages}{197--227}.
\bibitem[{Carvalho et~al.(2019)Carvalho, Pereira \& Cardoso}]{Carvalho2019}
\bibinfo{author}{Carvalho, D.~V.}, \bibinfo{author}{Pereira, E.~M.}, \&
  \bibinfo{author}{Cardoso, J.~S.} (\bibinfo{year}{2019}).
\newblock \bibinfo{title}{Machine learning interpretability: A survey on
  methods and metrics}.
\newblock {\it \bibinfo{journal}{Electronics}\/},  {\it \bibinfo{volume}{8}\/},
  \bibinfo{pages}{1--34}.
\bibitem[{Chen et~al.(2021)Chen, Lundberg \& Lee}]{Chen2021}
\bibinfo{author}{Chen, H.}, \bibinfo{author}{Lundberg, S.}, \&
  \bibinfo{author}{Lee, S.-I.} (\bibinfo{year}{2021}).
\newblock \bibinfo{title}{Explaining models by propagating {S}hapley values of
  local components}.
\newblock In {\it \bibinfo{booktitle}{Explainable {AI} in Healthcare and
  Medicine}\/} (pp. \bibinfo{pages}{261--270}).
\newblock \bibinfo{publisher}{Springer, Cham}.
\bibitem[{Choquet(1954)}]{Choquet1954}
\bibinfo{author}{Choquet, G.} (\bibinfo{year}{1954}).
\newblock \bibinfo{title}{Theory of capacities}.
\newblock {\it \bibinfo{journal}{Annales de l'Institut Fourier}\/},  {\it
  \bibinfo{volume}{5}\/}, \bibinfo{pages}{131--295}.
\bibitem[{Cortez et~al.(2009)Cortez, Cerdeira, Almeida, Matos \&
  Reis}]{Cortez2009}
\bibinfo{author}{Cortez, P.}, \bibinfo{author}{Cerdeira, A.},
  \bibinfo{author}{Almeida, F.}, \bibinfo{author}{Matos, T.}, \&
  \bibinfo{author}{Reis, J.} (\bibinfo{year}{2009}).
\newblock \bibinfo{title}{{Modeling wine preferences by data mining from
  physicochemical properties}}.
\newblock {\it \bibinfo{journal}{Decision Support Systems}\/},  {\it
  \bibinfo{volume}{47}\/}, \bibinfo{pages}{547--553}.
\bibitem[{Efron et~al.(2004)Efron, Hastie, Johnstone \& Tibshirani}]{Efron2004}
\bibinfo{author}{Efron, B.}, \bibinfo{author}{Hastie, T.},
  \bibinfo{author}{Johnstone, I.}, \& \bibinfo{author}{Tibshirani, R.}
  (\bibinfo{year}{2004}).
\newblock \bibinfo{title}{{Least angle regression}}.
\newblock {\it \bibinfo{journal}{The Annals of Statistics}\/},  {\it
  \bibinfo{volume}{32}\/}, \bibinfo{pages}{407--499}.
\bibitem[{Fawagreh et~al.(2014)Fawagreh, Gaber \& Elyan}]{Fawagreh2014}
\bibinfo{author}{Fawagreh, K.}, \bibinfo{author}{Gaber, M.~M.}, \&
  \bibinfo{author}{Elyan, E.} (\bibinfo{year}{2014}).
\newblock \bibinfo{title}{Random forests: From early developments to recent
  advancements}.
\newblock {\it \bibinfo{journal}{Systems Science and Control Engineering}\/},
  {\it \bibinfo{volume}{2}\/}, \bibinfo{pages}{602--609}.
\bibitem[{Garreau \& von Luxburg(2020)}]{Garreau2020}
\bibinfo{author}{Garreau, D.}, \& \bibinfo{author}{von Luxburg, U.}
  (\bibinfo{year}{2020}).
\newblock \bibinfo{title}{Looking deeper into tabular lime}.
\newblock {\it \bibinfo{journal}{ArXiv ID: 2008.11092}\/}, . \URLprefix
  \url{http://arxiv.org/abs/2008.11092}.
\bibitem[{Gilpin et~al.(2018)Gilpin, Bau, Yuan, Bajwa, Specter \&
  Kagal}]{Gilpin2018}
\bibinfo{author}{Gilpin, L.~H.}, \bibinfo{author}{Bau, D.},
  \bibinfo{author}{Yuan, B.~Z.}, \bibinfo{author}{Bajwa, A.},
  \bibinfo{author}{Specter, M.}, \& \bibinfo{author}{Kagal, L.}
  (\bibinfo{year}{2018}).
\newblock \bibinfo{title}{Explaining explanations: An overview of
  interpretability of machine learning}.
\newblock In {\it \bibinfo{booktitle}{2018 IEEE 5th International Conference on
  Data Science and Advanced Analytics (DSAA 2018)}\/} (pp.
  \bibinfo{pages}{80--89}).
\newblock \bibinfo{publisher}{IEEE}.
\bibitem[{Goodfellow et~al.(2016)Goodfellow, Bengio \&
  Courville}]{Goodfellow2016}
\bibinfo{author}{Goodfellow, I.}, \bibinfo{author}{Bengio, Y.}, \&
  \bibinfo{author}{Courville, A.} (\bibinfo{year}{2016}).
\newblock {\it \bibinfo{title}{Deep learning}\/}.
\newblock \bibinfo{publisher}{MIT Press}.
\bibitem[{Grabisch(1996)}]{Grabisch1996}
\bibinfo{author}{Grabisch, M.} (\bibinfo{year}{1996}).
\newblock \bibinfo{title}{The application of fuzzy integrals in multicriteria
  decision making}.
\newblock {\it \bibinfo{journal}{European Journal of Operational Research}\/},
  {\it \bibinfo{volume}{89}\/}, \bibinfo{pages}{445--456}.
\bibitem[{Grabisch(1997{\natexlab{a}})}]{Grabisch1997a}
\bibinfo{author}{Grabisch, M.} (\bibinfo{year}{1997}{\natexlab{a}}).
\newblock \bibinfo{title}{Alternative representations of discrete fuzzy
  measures for decision making}.
\newblock {\it \bibinfo{journal}{International Journal of Uncertainty Fuzziness
  and Knowledge-Based Systems}\/},  {\it \bibinfo{volume}{5}\/},
  \bibinfo{pages}{587--607}.
\bibitem[{Grabisch(1997{\natexlab{b}})}]{Grabisch1997b}
\bibinfo{author}{Grabisch, M.} (\bibinfo{year}{1997}{\natexlab{b}}).
\newblock \bibinfo{title}{$k$-order additive discrete fuzzy measures and their
  representation}.
\newblock {\it \bibinfo{journal}{Fuzzy Sets and Systems}\/},  {\it
  \bibinfo{volume}{92}\/}, \bibinfo{pages}{167--189}.
\bibitem[{Grabisch(2016)}]{Grabisch2016}
\bibinfo{author}{Grabisch, M.} (\bibinfo{year}{2016}).
\newblock {\it \bibinfo{title}{Set Functions, games and capacities in decision
  making}\/}.
\newblock \bibinfo{publisher}{Springer International Publishing}.
\bibitem[{Grabisch et~al.(2002)Grabisch, Duchêne, Lino \&
  Perny}]{Grabisch2002}
\bibinfo{author}{Grabisch, M.}, \bibinfo{author}{Duchêne, J.},
  \bibinfo{author}{Lino, F.}, \& \bibinfo{author}{Perny, P.}
  (\bibinfo{year}{2002}).
\newblock \bibinfo{title}{Subjective evaluation of discomfort in sitting
  positions}.
\newblock {\it \bibinfo{journal}{Fuzzy Optimization and Decision Making}\/},
  {\it \bibinfo{volume}{1}\/}, \bibinfo{pages}{287--312}.
\bibitem[{Grabisch \& Labreuche(2010)}]{Grabisch2010}
\bibinfo{author}{Grabisch, M.}, \& \bibinfo{author}{Labreuche, C.}
  (\bibinfo{year}{2010}).
\newblock \bibinfo{title}{A decade of application of the {C}hoquet and sugeno
  integrals in multi-criteria decision aid}.
\newblock {\it \bibinfo{journal}{Annals of Operations Research}\/},  {\it
  \bibinfo{volume}{175}\/}, \bibinfo{pages}{247--286}.
\bibitem[{Grabisch et~al.(2006)Grabisch, Prade, Raufaste \&
  Terrier}]{Grabisch2006}
\bibinfo{author}{Grabisch, M.}, \bibinfo{author}{Prade, H.},
  \bibinfo{author}{Raufaste, E.}, \& \bibinfo{author}{Terrier, P.}
  (\bibinfo{year}{2006}).
\newblock \bibinfo{title}{Application of the {C}hoquet integral to subjective
  mental workload evaluation}.
\newblock {\it \bibinfo{journal}{IFAC Proceedings Volumes}\/},  {\it
  \bibinfo{volume}{39}\/}, \bibinfo{pages}{135--140}.
\bibitem[{Kruppa et~al.(2013)Kruppa, Schwarz, Arminger \& Ziegler}]{Kruppa2013}
\bibinfo{author}{Kruppa, J.}, \bibinfo{author}{Schwarz, A.},
  \bibinfo{author}{Arminger, G.}, \& \bibinfo{author}{Ziegler, A.}
  (\bibinfo{year}{2013}).
\newblock \bibinfo{title}{Consumer credit risk: Individual probability
  estimates using machine learning}.
\newblock {\it \bibinfo{journal}{Expert Systems with Applications}\/},  {\it
  \bibinfo{volume}{40}\/}, \bibinfo{pages}{5125--5131}.
\bibitem[{LeCun et~al.(2015)LeCun, Bengio \& Hinton}]{LeCun2015}
\bibinfo{author}{LeCun, Y.}, \bibinfo{author}{Bengio, Y.}, \&
  \bibinfo{author}{Hinton, G.} (\bibinfo{year}{2015}).
\newblock {\it \bibinfo{title}{Deep learning}\/} volume \bibinfo{volume}{521}.
\bibitem[{Lipton(2018)}]{Lipton2018}
\bibinfo{author}{Lipton, Z.~C.} (\bibinfo{year}{2018}).
\newblock \bibinfo{title}{The mythos of machine learning interpretability}.
\newblock {\it \bibinfo{journal}{Machine Learning}\/},  {\it
  \bibinfo{volume}{16}\/}, \bibinfo{pages}{31--57}.
\bibitem[{Lundberg et~al.(2020)Lundberg, Erion, Chen, DeGrave, Prutkin, Nair,
  Katz, Himmelfarb, Bansal \& Lee}]{Lundberg2020}
\bibinfo{author}{Lundberg, S.~M.}, \bibinfo{author}{Erion, G.},
  \bibinfo{author}{Chen, H.}, \bibinfo{author}{DeGrave, A.},
  \bibinfo{author}{Prutkin, J.~M.}, \bibinfo{author}{Nair, B.},
  \bibinfo{author}{Katz, R.}, \bibinfo{author}{Himmelfarb, J.},
  \bibinfo{author}{Bansal, N.}, \& \bibinfo{author}{Lee, S.-I.}
  (\bibinfo{year}{2020}).
\newblock \bibinfo{title}{From local explanations to global understanding with
  explainable {AI} for trees}.
\newblock {\it \bibinfo{journal}{Nature Machine Intelligence}\/},  {\it
  \bibinfo{volume}{2}\/}, \bibinfo{pages}{56--67}.
\bibitem[{Lundberg \& Lee(2017)}]{Lundberg2017}
\bibinfo{author}{Lundberg, S.~M.}, \& \bibinfo{author}{Lee, S.-I.}
  (\bibinfo{year}{2017}).
\newblock \bibinfo{title}{A unified approach to interpreting model
  predictions}.
\newblock In \bibinfo{editor}{I.~Guyon}, \bibinfo{editor}{U.~V. Luxburg},
  \bibinfo{editor}{S.~Bengio}, \bibinfo{editor}{H.~Wallach},
  \bibinfo{editor}{R.~Fergus}, \bibinfo{editor}{S.~Vishwanathan}, \&
  \bibinfo{editor}{R.~Garnett} (Eds.), {\it \bibinfo{booktitle}{Advances in
  Neural Information Processing Systems 30}\/} (pp.
  \bibinfo{pages}{4765--4774}).
\bibitem[{Lundberg et~al.(2018)Lundberg, Nair, Vavilala, Horibe, Eisses, Adams,
  Liston, Low, Newman, Kim \& Lee}]{Lundberg2018}
\bibinfo{author}{Lundberg, S.~M.}, \bibinfo{author}{Nair, B.},
  \bibinfo{author}{Vavilala, M.~S.}, \bibinfo{author}{Horibe, M.},
  \bibinfo{author}{Eisses, M.~J.}, \bibinfo{author}{Adams, T.},
  \bibinfo{author}{Liston, D.~E.}, \bibinfo{author}{Low, D. K.~W.},
  \bibinfo{author}{Newman, S.~F.}, \bibinfo{author}{Kim, J.}, \&
  \bibinfo{author}{Lee, S.~I.} (\bibinfo{year}{2018}).
\newblock \bibinfo{title}{Explainable machine-learning predictions for the
  prevention of hypoxaemia during surgery}.
\newblock {\it \bibinfo{journal}{Nature Biomedical Engineering}\/},  {\it
  \bibinfo{volume}{2}\/}, \bibinfo{pages}{749--760}.
\bibitem[{Miller(2019)}]{Miller2019}
\bibinfo{author}{Miller, T.} (\bibinfo{year}{2019}).
\newblock \bibinfo{title}{Explanation in artificial intelligence: Insights from
  the social sciences}.
\newblock {\it \bibinfo{journal}{Artificial Intelligence}\/},  {\it
  \bibinfo{volume}{267}\/}, \bibinfo{pages}{1--38}.
\bibitem[{Molnar(2021)}]{Molnar2021}
\bibinfo{author}{Molnar, C.} (\bibinfo{year}{2021}).
\newblock {\it \bibinfo{title}{Interpretable machine learning}\/}.
\newblock \URLprefix \url{https://christophm.github.io/interpretable-ml-book/}.
\bibitem[{Murofushi \& Soneda(1993)}]{Murofushi1993}
\bibinfo{author}{Murofushi, T.}, \& \bibinfo{author}{Soneda, S.}
  (\bibinfo{year}{1993}).
\newblock \bibinfo{title}{Techniques for reading fuzzy measures (iii):
  interaction index}.
\newblock In {\it \bibinfo{booktitle}{9th fuzzy system symposium}\/} (pp.
  \bibinfo{pages}{693--696}).
\bibitem[{Pedregosa et~al.(2011)Pedregosa, Varoquaux, Gramfort, Michel,
  Thirion, Grisel, Blondel, Prettenhofer, Weiss, Dubourg, Vanderplas, Passos,
  Cournapeau, Brucher, Perrot \& Duchesnay}]{Pedregosa2011}
\bibinfo{author}{Pedregosa, F.}, \bibinfo{author}{Varoquaux, G.},
  \bibinfo{author}{Gramfort, A.}, \bibinfo{author}{Michel, V.},
  \bibinfo{author}{Thirion, B.}, \bibinfo{author}{Grisel, O.},
  \bibinfo{author}{Blondel, M.}, \bibinfo{author}{Prettenhofer, P.},
  \bibinfo{author}{Weiss, R.}, \bibinfo{author}{Dubourg, V.},
  \bibinfo{author}{Vanderplas, J.}, \bibinfo{author}{Passos, A.},
  \bibinfo{author}{Cournapeau, D.}, \bibinfo{author}{Brucher, M.},
  \bibinfo{author}{Perrot, M.}, \& \bibinfo{author}{Duchesnay, E.}
  (\bibinfo{year}{2011}).
\newblock \bibinfo{title}{{Scikit-learn: Machine Learning in {P}ython}}.
\newblock {\it \bibinfo{journal}{Journal of Machine Learning Research}\/},
  {\it \bibinfo{volume}{12}\/}, \bibinfo{pages}{2825--2830}.
\bibitem[{Pelegrina et~al.(2020)Pelegrina, Duarte, Grabisch \&
  Romano}]{Pelegrina2020}
\bibinfo{author}{Pelegrina, G.~D.}, \bibinfo{author}{Duarte, L.~T.},
  \bibinfo{author}{Grabisch, M.}, \& \bibinfo{author}{Romano, J. M.~T.}
  (\bibinfo{year}{2020}).
\newblock \bibinfo{title}{The multilinear model in multicriteria decision
  making: The case of 2-additive capacities and contributions to parameter
  identification}.
\newblock {\it \bibinfo{journal}{European Journal of Operational Research}\/},
  {\it \bibinfo{volume}{282}\/}.
\bibitem[{Ribeiro et~al.(2016)Ribeiro, Singh \& Guestrin}]{Ribeiro2016}
\bibinfo{author}{Ribeiro, M.~T.}, \bibinfo{author}{Singh, S.}, \&
  \bibinfo{author}{Guestrin, C.} (\bibinfo{year}{2016}).
\newblock \bibinfo{title}{"{Why} should {I} trust you?" {E}xplaining the
  predictions of any classifier}.
\newblock In {\it \bibinfo{booktitle}{Proceedings of the 22nd ACM SIGKDD
  international conference on knowledge discovery and data mining}\/} (pp.
  \bibinfo{pages}{1135--1144}).
\bibitem[{Setzu et~al.(2021)Setzu, Guidotti, Monreale, Turini, Pedreschi \&
  Giannotti}]{Setzu2021}
\bibinfo{author}{Setzu, M.}, \bibinfo{author}{Guidotti, R.},
  \bibinfo{author}{Monreale, A.}, \bibinfo{author}{Turini, F.},
  \bibinfo{author}{Pedreschi, D.}, \& \bibinfo{author}{Giannotti, F.}
  (\bibinfo{year}{2021}).
\newblock \bibinfo{title}{Glocalx - from local to global explanations of black
  box {AI} models}.
\newblock {\it \bibinfo{journal}{Artificial Intelligence}\/},  {\it
  \bibinfo{volume}{294}\/}, \bibinfo{pages}{103457}.
\bibitem[{Shapley(1953)}]{Shapley1953}
\bibinfo{author}{Shapley, L.~S.} (\bibinfo{year}{1953}).
\newblock \bibinfo{title}{A value for n-person games}.
\newblock In \bibinfo{editor}{W.~Kuhn}, \& \bibinfo{editor}{A.~W. Tucker}
  (Eds.), {\it \bibinfo{booktitle}{Annals of mathematics studies: Vol. 28.
  Contributions to the theory of games, Vol. II}\/} (pp.
  \bibinfo{pages}{307--317}).
\newblock \bibinfo{publisher}{Princeton University Press}.
\bibitem[{Štrumbelj \& Kononenko(2010)}]{Strumbelj2010}
\bibinfo{author}{Štrumbelj, E.}, \& \bibinfo{author}{Kononenko, I.}
  (\bibinfo{year}{2010}).
\newblock \bibinfo{title}{An efficient explanation of individual
  classifications using game theory}.
\newblock {\it \bibinfo{journal}{Journal of Machine Learning Research}\/},
  {\it \bibinfo{volume}{11}\/}, \bibinfo{pages}{1--18}.
\bibitem[{Štrumbelj \& Kononenko(2014)}]{Strumbelj2014}
\bibinfo{author}{Štrumbelj, E.}, \& \bibinfo{author}{Kononenko, I.}
  (\bibinfo{year}{2014}).
\newblock \bibinfo{title}{Explaining prediction models and individual
  predictions with feature contributions}.
\newblock {\it \bibinfo{journal}{Knowledge and Information Systems}\/},  {\it
  \bibinfo{volume}{41}\/}, \bibinfo{pages}{647--665}.
\bibitem[{Xin et~al.(2018)Xin, Kong, Liu, Chen, Li, Zhu, Gao, Hou \&
  Wang}]{Xin2018}
\bibinfo{author}{Xin, Y.}, \bibinfo{author}{Kong, L.}, \bibinfo{author}{Liu,
  Z.}, \bibinfo{author}{Chen, Y.}, \bibinfo{author}{Li, Y.},
  \bibinfo{author}{Zhu, H.}, \bibinfo{author}{Gao, M.}, \bibinfo{author}{Hou,
  H.}, \& \bibinfo{author}{Wang, C.} (\bibinfo{year}{2018}).
\newblock \bibinfo{title}{Machine learning and deep learning methods for
  cybersecurity}.
\newblock {\it \bibinfo{journal}{IEEE Access}\/},  {\it \bibinfo{volume}{6}\/},
  \bibinfo{pages}{35365--35381}.
\bibitem[{Young(1985)}]{Young1985}
\bibinfo{author}{Young, H.~P.} (\bibinfo{year}{1985}).
\newblock \bibinfo{title}{Monotonic solutions of cooperative games}.
\newblock {\it \bibinfo{journal}{International Journal of Game Theory}\/},
  {\it \bibinfo{volume}{14}\/}, \bibinfo{pages}{65--72}.

\end{thebibliography}

\end{document}